\documentclass{article}
\PassOptionsToPackage{numbers}{natbib}
\usepackage[final,main]{neurips_2026} 
\usepackage{algorithm}
\usepackage{algpseudocode}
\usepackage{amsmath}
\usepackage{amsfonts}
\usepackage{amssymb}
\usepackage{amsthm}
\usepackage{pifont}
\usepackage{multirow}
\usepackage[utf8]{inputenc} 
\usepackage[T1]{fontenc}    
\usepackage{url}            
\usepackage{booktabs}       
\usepackage{microtype}      
\usepackage[table]{xcolor}  
\usepackage{graphicx}
\usepackage{caption}
\usepackage{hyperref}       

\theoremstyle{plain}
\newtheorem{theorem}{\indent Theorem}[section]

\newtheorem{assumption}{\indent Assumption}[section]

\newtheorem{definition}{\indent Definition}[section]

\newcommand{\cmark}{\ding{51}}
\newcommand{\xmark}{\ding{55}}
\newcommand{\greencmark}{\textcolor{green!50!black}{\cmark}}
\newcommand{\redxmark}{\textcolor{red!70!black}{\xmark}}

\begin{document}

\title{\textsc{ShapleyContextPruning}: A Cooperative Game Perspective for Context Reranking and Pruning}

\author{%
  Yanqiao Chen$^{*}$, Dongsheng Hou$^{*}$, Yuhan Rui$^{*}$, Zhen Cao, Yepang Liu$^{\dagger}$ \\
  Southern University of Science and Technology \\
  \texttt{\{12412115, 12410421\}@mail.sustech.edu.cn} \\
  \texttt{ruiyuhan0110@gmail.com}, \texttt{12410102@mail.sustech.edu.cn} \\
  \texttt{liuyp1@sustech.edu.cn} \\
  $^*$Equal contribution. \quad $^{\dagger}$Corresponding author.
}

\maketitle

\begin{abstract}
Context reranking and pruning have become essential for improving the efficiency of modern \textbf{Retrieval-Augmented Generation (RAG)} systems, yet an \textbf{interpretable and unified framework} remains underexplored. Previous work has primarily emphasized lexical retrieval, cross-encoder architectures, model distillation, and Low-Rank Adaptation (LoRA), mostly relying on heuristic loss functions and empirical attribution. This paper presents \textbf{Shapley Context Pruning (SCP)}, a novel framework for context reranking that establishes a \textbf{cooperative-game-theory perspective} for importance attribution by modeling the context as a cooperative game. Balancing the trade-off between fine-grained and coarse-grained representations, we employ a Deep Sets architecture to approximate a permutation-invariant value function at the sentence level, utilizing pre-trained language models as sentence embedders and optimizing via a pairwise margin ranking loss. To ensure practical scalability without sacrificing mathematical rigor, we leverage \textbf{Monte-Carlo sampling} for efficient training and inference, providing formal \textbf{theoretical error bounds} and \textbf{sample complexity guarantees} for preserving Top-$K$ subset rankings. Furthermore, we conduct \textbf{comprehensive experiments}---spanning supporting-sentence recall, Needle-in-the-Haystack (NIAH) evaluations, long-context QA, and multi-hop reasoning---alongside rigorous ablation studies on embedding quality and attribution strategies. The model achieves competitive downstream QA performance against robust baselines. Notably, our framework decouples the attribution logic into an exceptionally lightweight \textbf{3M-parameter value network} prototype that operates on top of standard sentence embeddings, underscoring its modular efficiency and potential deployment feasibility as a scalable context filter. Additionally, we provide multiple case studies, including robustness justifications, analyzing real-world application scenarios. Beyond practical contributions, we propose an intuitive theoretical blueprint for context analysis---the \emph{``landscape of the context''}---which offers a structural perspective on information aggregation. Finally, based on our empirical observations, we formulate several hypotheses and open research questions for future investigation.
\end{abstract}

\section{Introduction}

\begin{figure} 
\centering 
\includegraphics[width=0.95\textwidth]{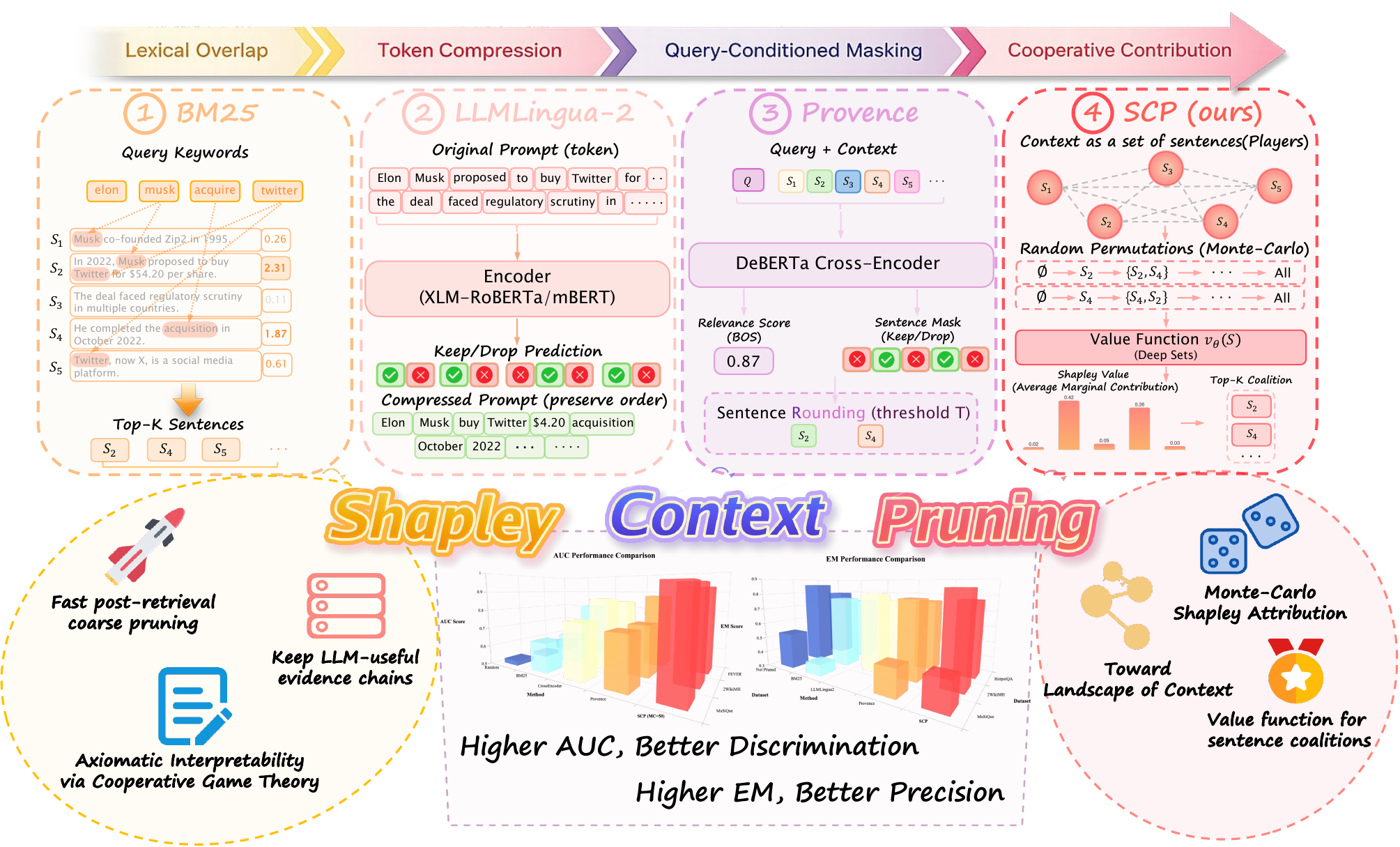}
\caption{An overview of concept evolution and how SCP different from the previous methods. SCP explicitly models interactions between components as a cooperative game rather than simple sets or sequences. Inspired by explainable AI (XAI) theory, SCP extend the restrictive methodology to a wider range of scenarios, i.e., context reranking.}
\label{fig:teaser}
\end{figure}

\begin{figure} 
\centering
\includegraphics[width=0.95\textwidth]{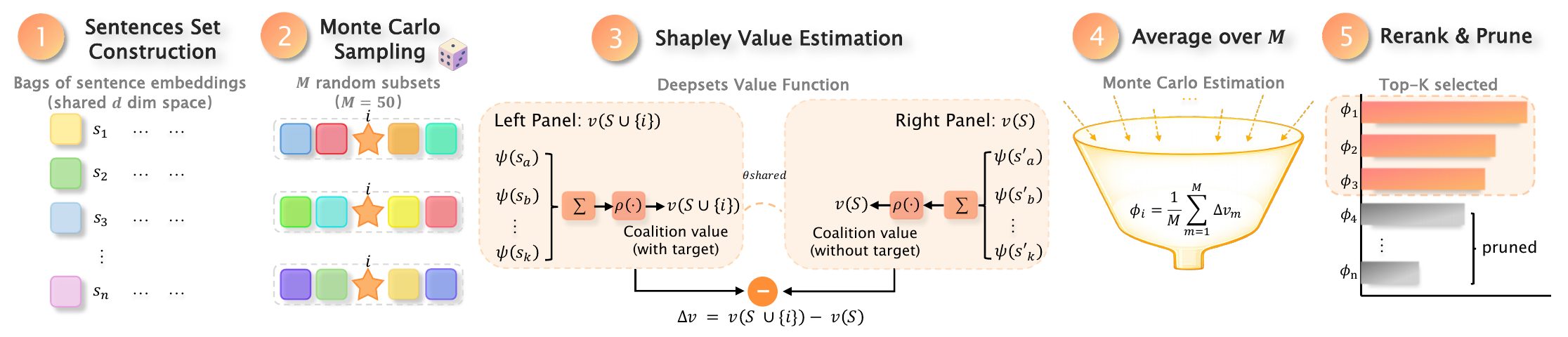}
\caption{Architecture Overview. The SCP framework consists of three main components: (i) Embedder: Can be any pre-trained embedding model. (ii) Value Function: Evaluating the value of a subset of sentences. (iii) Shapley Estimator: Estimating the Shapley value of each sentence based on the value function. The output is a ranked list of sentences based on their Shapley values, which can be used for pruning. Despite post-hoc explanation in XAI, SCP actively learns a value function that directly optimizes the ranking of supporting sentences, thus ensuring that the Shapley values are meaningful for the downstream task.}
\label{fig:archi}
\end{figure}

Retrieval-Augmented Generation (RAG) \cite{lewis2021retrievalaugmentedgenerationknowledgeintensivenlp} has emerged as a dominant paradigm for mitigating hallucinations and expanding the knowledge boundary of Large Language Models (LLMs). Modern RAG systems generally partition retrieval, reranking, and pruning into two stages: first, applying a coarse-grained model (e.g., BM25 \cite{robertson2009bm25}, Bi-Encoders \cite{reimers2019sentencebertsentenceembeddingsusing}) to efficiently filter candidate components; second, utilizing a fine-grained model (e.g., Cross-Encoders \cite{nogueira2020passagererankingbert}) to rerank and prune. However, as retrieved documents scale in length and complexity, processing expansive contexts incurs prohibitively high computational costs and often distracts models with irrelevant noise. Existing approaches frequently cascade language models with Low-Rank Adaptation (LoRA) or distillation techniques; however, their loss functions and attribution mechanisms predominantly rely on heuristics and empirical observations. Consequently, the absence of a systematic theoretical background leads to empirically driven designs, constraining the analytical rigor essential for transparent architectural improvements. In a recent survey \cite{mei2025surveycontextengineeringlarge}, researchers proposed an abstract theoretical formulation for context modeling, wherein context components are aggregated via an assembly function \(\mathcal{A}(c_1, c_2, \ldots, c_n)\). The objective is to identify a context-generating function that maximizes the expected quality of the LLM's ensuing output. This formulation intuitively parallels the concept of a ``coalition'' in game theory, directly motivating our integration of the Shapley value into the reranking stage.

In this paper, we propose Shapley Context Pruning (SCP), a novel framework that transcends the interpretability limitations of existing heuristic methods and leverages cooperative game theory to improve coarse-grained reranking accuracy. Our primary contributions are summarized as follows:

\paragraph{A Novel Attribution Perspective for Context Engineering} Deviating from conventional heuristic attribution techniques, we introduce the Shapley value at the sentence level to establish a robust, theory-inspired attribution mechanism. We naturally adopt a setwise value function with pairwise ranking supervision and listwise pruning output with the help of the Shapley value. Beyond implementation, this work provides a conceptual way to bridge engineering and theory via cooperative game theory. See Figure~\ref{fig:teaser} for overview. We discuss the broader vision and future directions for context engineering in Section~\ref{sec:shap_all}.

\paragraph{A Coarse-Grained Reranker Prototype for RAG Systems} We propose an operational Shapley-value-based reranking and pruning framework tailored for the preliminary context pruning stage—where the raw text is exceedingly long and demands an aggressive, coarse-grained filter. By reranking sentences according to their computed Shapley values, the mechanism systematically retains the most informative ``coalition'' of sentences that maximally contribute to downstream generation tasks. This approach offers a novel, highly scalable perspective on the efficiency--accuracy trade-off, prioritizing structural integrity over merely pursuing state-of-the-art benchmarks. See Figure~\ref{fig:archi} for the architecture overview.

\paragraph{Engineering Implementation and Comprehensive Experimental Validation} We validate SCP across a spectrum of rigorous benchmarks, encompassing supporting-sentence recall, Needle-in-the-Haystack (NIAH) tasks, long-context QA, and multi-hop reasoning. Our results confirm that SCP maintains competitive overall performance, validating its feasibility on datasets such as MuSiQue, 2WikiMH, and HotpotQA. To further validate our approach, we present extensive ablation studies analyzing the impacts of embedding choices, Monte-Carlo sampling budgets, and comparative attribution strategies against classical methods (e.g., Leave-One-Out). Additionally, we provide case studies and empirical robustness analyses. Crucially, to isolate and certify the efficacy of SCP ablation, we conduct evaluations independent of end-to-end RAG confounders.

Ultimately, the SCP framework introduces a versatile methodology applicable beyond context pruning—extending to general component reranking, feature selection, and software test-case optimization. By integrating rigorous theoretical grounding with scalable implementation, this framework contributes a robust analytical tool for the broader scientific community.

\section{Related Work}

Recent context pruning methods fall into four categories. \emph{(i)} Token-level prompt compression (LLMLingua-2 \cite{jiang2023llmlingua,pan2024llmlingua2} and Selective Context \cite{li2023compressing}) compresses prompts via surprisal-based selection, but relies on local token signals without cross-sentence cooperative modeling; we use LLMLingua-2 as our baseline. \emph{(ii)} BERT-based sentence-level pruning (Provence \cite{chirkova2025provence}) filters content using transformer-style adaptation. \emph{(iii)} Extractive and abstractive compressors (RECOMP \cite{xu2023recompimprovingretrievalaugmentedlms}) condense passages into new representations, diverging from strict sentence-selection reranking. \emph{(iv)} Task-specialized systems (SWE-Pruner \cite{wang2026swepruner}) show that structure-aware pruning yields gains in domain-specific settings, emphasizing the need to align pruning with downstream task structure.

In standard RAG pipelines, retrieval and reranking dictate the effective context length. Sparse retrieval (BM25 \cite{robertson2009bm25}) remains a strong lexical baseline, while dense retrieval (DPR \cite{karpukhin2020dpr}) and cross-encoder rerankers \cite{nogueira2020passagererankingbert} improve semantic matching at higher computational cost. Models like FiD \cite{izacard2021fid} demonstrate that inter-passage interactions are critical for multi-hop reasoning. Meanwhile, ``Lost in the Middle'' \cite{liu2024lost} and related surveys \cite{liu2025survey} reveal that LLMs underutilize long contexts, motivating selective, high-quality pruning before generation. Feature attribution studies also highlight a gap between local relevance and holistic contribution: attention weights \cite{clark2019bert} do not equate to feature importance, and gradient methods such as Integrated Gradients \cite{sundararajan2017ig} lack direct mappings to discrete, set-valued context selection.

Learning-to-rank work spans pointwise \cite{10.1007/11776420_44}, pairwise \cite{819548}, setwise \cite{Zhuang_2024}, and listwise \cite{10.1145/1273496.1273513} paradigms. OptiSet \cite{jiang2026optisetunifiedoptimizingset} pursues set-wise modeling by selecting optimal component combinations, aligning with our objective of optimizing discrete coalitions. Closer to our setting, LooComp \cite{do2026loocompleverageleaveoneoutstrategy} applies a Leave-One-Out (LOO) objective with LoRA for context pruning. We differ in two key ways: (i) we train an independent 3M-parameter Deep Sets value function from scratch, decoupling attribution from the embedder; (ii) our analysis (Appendix~\ref{sec:LOO}) shows LOO collapses under semantic redundancy, assigning near-zero scores to duplicated sentences, whereas Shapley attribution---averaging marginal contributions across all coalitions---inherently resists this failure mode. Applying Shapley attribution to subsets requires a permutation-invariant value function, for which Deep Sets \cite{zaheer2018deepsets} is the canonical choice; we keep the estimator lightweight to avoid representation bottlenecks under finite latent dimensions \cite{wagstaff2019limitations}.

The Shapley value \cite{shapley1953value}, originating in cooperative game theory, quantifies a component's expected marginal contribution across all coalitions---an attribution scheme that aligns naturally with context pruning, where a sentence inconsequential in isolation may prove critical when synergized with bridging evidence. Yet existing Shapley-based frameworks each sacrifice one key property: Data Shapley \cite{ghorbani2019datashapley} achieves interpretability and redundancy robustness but is fundamentally unscalable, since each evaluation of $v(S)$ requires retraining the underlying model; SHAP \cite{lundberg2017shap} explains a fixed model by attributing along feature dimensions rather than samples, and thus does not rank context units; TokenSHAP \cite{goldshmidt2024tokenshapinterpretinglargelanguage} operates at an overly fine token granularity with a hard-coded value function (e.g., TF-IDF cosine similarity), discarding richer semantic signals. SCP closes this gap by \emph{learning} a dedicated, permutation-invariant value function at the sentence level, jointly delivering scalability, interpretability, and redundancy robustness. Beyond pruning, this learned value function opens new directions for context engineering---hierarchical context structures, bridging-sentence analysis, and tree-based context management---which we explore in Appendix~\ref{sec:shap_all}.

In summary, while SCP builds upon established theoretical ingredients---Shapley values, Deep Sets, and cooperative games---its contribution lies in answering, for the first time, how Shapley attribution can be feasibly operationalized for context reranking, and in providing a novel blueprint toward more systematic and interpretable context pruning and management. See Table~\ref{tab:diff_shapley} for detailed comparison.

\section{Shapley Context Pruning (SCP)} 

\subsection{The Shapley value} 

We introduce the Shapley value \cite{shapley1953value} from cooperative game theory as the theoretical foundation of our context pruning method. The Shapley value is a solution concept in cooperative game theory that aims to fairly distribute the total gains (or costs) among players based on their individual contributions to the overall outcome.

\begin{definition}[Shapley Value]
Consider a cooperative game with a set of players \(N\) and a characteristic 
function \(v
: 2^N \rightarrow \mathbb{R}\) that assigns a value to each coalition of 
players. The Shapley value \(\phi_i(v)\) for player \(i \in N\) is defined 
as:
\[
\phi_i(v) = \sum_{S \subseteq N \setminus \{i\}} \frac{|S|!(|N| - |S| - 1)!}{|N|!} [v(S \cup \{i\}) - v(S)]
\]
where the sum is taken over all subsets \(S\) of \(N\) that do not include 
player \(i\).
\end{definition}

\paragraph{Rationale} We draw inspiration from the Shapley value to design a ranking-based importance estimator. Although our learned value function $v_\theta$ does not strictly satisfy all Shapley axioms due to its data-driven rather than axiomatic construction, the Shapley formulation serves as the conceptual foundation for our coalition-based importance aggregation. While we relax these axiomatic constraints to maximize representational capacity, reintroducing specific theoretical constraints via targeted training regularization presents a promising avenue for future research to bridge empirical flexibility with formal guarantees. See Section~\ref{sec:shap} for a formal description of the four Shapley properties and how these properties correspond to real context reranking. 

\subsection{The Value Function} 

Constructing an accurate and tractable value function \(v(S)\) is critical for deriving Shapley values. Specifically, we use a supervised learning paradigm to map any chosen subset of contextual passages to an expected retrieval/generation performance metric (e.g., QA accuracy or span recall). Given that $S$ is an unordered set, the value function should satisfy permutation invariance. We therefore adopt the Deep Sets architecture.

\begin{definition}[Deep Sets \cite{zaheer2018deepsets}]
A function \(f: 2^{\mathcal{X}} \to \mathcal{Y}\) acting on sets is 
permutation invariant if \(f(\{x_1, \ldots, x_M\}) = f(\{x_{\pi(1)}, \ldots, x_{\pi(M)}\})\) for any permutation \(\pi\). The Deep Sets theorem states that a function \(f\) operating on a set \(X\) having permutation invariance can be universally approximated by decomposing into the form:
\[ f(X) = \rho \left( \sum_{x \in X} \psi(x) \right) \]
where \(\psi: \mathcal{X} \to \mathbb{R}^d\) and \(\rho: \mathbb{R}^d \to \mathcal{Y}\) are continuous transformations.
\end{definition}

In our context, the parameterized value function operates as $v(S; \theta) = \rho \left( \sum_{c_i \in S} \psi(c_i) \right) $ where the feature extractor \(\psi\) independently maps each context component \(c_i\) (the original sentence and the query) to a latent embedding, and the regressor \(\rho\) maps the additively aggregated representations to a scalar performance value. Both components are Multi-Layer Perceptrons (MLPs) parameterized by \(\theta\).

The training paradigm is structured as follows: \emph{(i)}Input: The embedded sentence set \(S \subseteq C\) paired with a specific reasoning query. \emph{(ii)}Output: A scalar prediction indicating the inferred value (relative importance) of the subset $S$ for answering the given query. Since obtaining precise real-valued targets requires excessive LLM queries, we supervise the optimization using a pairwise margin ranking loss. This objective only forces the network to correctly rank optimal evidence passages above distractors, rather than memorize exact performance figures:
\[
L(\theta) = \frac{1}{|P| \cdot |N|} \sum_{i \in P} \sum_{j \in N} \max(0, \epsilon - (\phi_i(v_\theta) - \phi_j(v_\theta)))
\]
where $P = \{k : y_k = 1\}$ is the set of positive supporting facts, $N = \{k : y_k = 0\}$ is the set of negative distractors, and $\epsilon > 0$ is the margin hyperparameter. In real training datasets, these sentences are annotated by human annotators. This formulation directly binds the Shapley value calculation into the backward pass, ensuring that relevant documents maintain a Shapley value margin of $\epsilon$ over irrelevant ones: $\phi_i > \phi_j$ whenever $y_i > y_j$. See Section~\ref{sec:per_inv} for justification.

\subsection{Efficient Shapley Estimation via Monte-Carlo Sampling}

Scaling exact Shapley value derivation poses substantial latency challenges for long-context scenarios, as it mandates $\mathcal{O}(2^n)$ model forward passes. To mitigate this bottleneck, we use an optimized Monte-Carlo sampling approximation \cite{castro2009polynomial}.

Rather than exhaustively querying the combinatorial space, the algorithm uniformly samples independent context permutations and accumulates the marginal delta contributions dynamically. We provide Algorithm \ref{alg:mc} in pseudocode form for clarity.

\begin{algorithm}
\captionsetup{labelformat=empty}
\caption{\colorbox{yellow!15}{\parbox{\dimexpr\linewidth-2\fboxsep\relax}{\textbf{Algorithm~\thealgorithm} Monte-Carlo Shapley Estimation}}}
\label{alg:mc}
\begin{algorithmic}[1]
\Require A context \(C = \{c_1, \ldots, c_n\}\), optimized value function \(v_\theta\), permutation count \(M\)
\Ensure Estimated Shapley values \(\{\hat{\phi}_1, \hat{\phi}_2, \ldots, \hat{\phi}_n\}\)
\State Initialize Shapley estimates \(\hat{\phi}_i = 0\) for all \(i \in \{1, \dots, n\}\)
\For{each sample \(m = 1\) to \(M\)}
    \State Sample a uniform random permutation \(\pi\) of the set \(C\)
    \State Initialize an expanding coalition subset \(S = \emptyset\)
    \For{each component \(c_{\pi(j)}\) according to \(\pi\)}
        \State Compute marginal contribution: \(\Delta v \gets v_\theta(S \cup \{c_{\pi(j)}\}) - v_\theta(S)\)
        \State Accumulate onto attribution: \(\hat{\phi}_{\pi(j)} \gets \hat{\phi}_{\pi(j)} + \frac{\Delta v}{M}\)
        \State Augment coalition: \(S \gets S \cup \{c_{\pi(j)}\}\)
    \EndFor
\EndFor
\State \textbf{return} \(\{\hat{\phi}_1, \hat{\phi}_2, \ldots, \hat{\phi}_n\}\)
\end{algorithmic}
\end{algorithm}

This estimator converges to the analytical Shapley value as the number of sampled
permutations increases, while capping forward evaluations at $\mathcal{O}(Mn)$. We prove the following error bound for the Shapley value estimation, under the assumption stated below. See Appendix A.4 for the proof.

\begin{assumption}[Bounded Value Function]
\label{ass:bounded}
The trained value function $v_{\theta}$ and the true value function $v^*$ are uniformly bounded: there exists a constant $B > 0$ such that for every coalition $S \subseteq N$,
\[
0 \leq v_{\theta}(S) \leq B \quad \text{and} \quad 0 \leq v^*(S) \leq B.
\]
In practice, this can be enforced by a sigmoid output layer or by clipping the predictions to $[0, B]$.
\end{assumption}

\begin{theorem}[Error Bound for Pruning Attribution]
\label{thm:error-bound}
Let $v^*$ be the true value function and $v_{\theta}$ a trained approximation satisfying Assumption~\ref{ass:bounded}. Assume the point-wise approximation error is uniformly bounded by $\epsilon_{\mathrm{approx}}$, i.e.,$|v^*(S) - v_{\theta}(S)| \leq \epsilon_{\mathrm{approx}}, \quad \forall S \subseteq N.$ Let $\hat{\phi}_i$ be the Monte-Carlo Shapley estimate of player $i$ obtained from $M$ independent uniform permutations (Algorithm~1). Then, with probability at least $1-\delta$ over the random permutations, the following holds simultaneously for all $i \in N$:
\[
\bigl|\hat{\phi}_i(v_{\theta}) - \phi_i(v^*)\bigr| \;\leq\; B\sqrt{\frac{2\ln(2n/\delta)}{M}} \;+\; 2\epsilon_{\mathrm{approx}}.
\]
\end{theorem}

The theorem shows that the training error and the estimation error are decoupled: they are additive, not multiplicative. See~\ref{fig:error_bound} for a visualization. From the error-bound theorem, as $n$ is around 10 for multi-hop reasoning tasks, we apply $M = 50$ for efficient training and inference. However, higher sampling counts are recommended. See the ablation study for further details.

\subsection{The General Shapley Context Pruning Algorithms}

The objective of context pruning is to preserve the useful components and filter the distractors.

\begin{definition}[Context Pruning Problem (sentence-level)] Given a set of context components (e.g., sentences or paragraphs) \(C = \{c_1, c_2, \ldots, c_n\}\), the objective is to learn a binary decision vector $D \in \{0,1\}^n$, where $D_i = 1$ denotes retaining the component $c_i$ and $D_i = 0$ denotes discarding it.
\end{definition} 

Once the robust value function $v_\theta$ is empirically optimized, calculating the contribution of any constituent part relies on establishing the Shapley value $\phi_i(v_{\theta})$. Subsequently, pruning reduces to a straightforward filtration process based on these generated utility scores. We outline the Shapley Context Pruning procedure via empirical criteria: an absolute top-K constraint depicted in Algorithm \ref{alg:shapley_pruning_topk}.

\paragraph{Rationale} Since the algorithm operates at a coarse-grained level, the pruner must reliably filter out the most irrelevant sentences while retaining the most informative ones; we therefore use a less aggressive pruning ratio for this preliminary stage. We provide a bound for Monte-Carlo sampling number for Top-K pruning in the appendix. See Theorem~\ref{thm:ranking-preservation}.

\section{Experiment} 

Our experiments are designed as feasibility studies: we seek to validate that a cooperative-game-theoretic formulation can produce meaningful, computationally tractable context rankings, rather than to claim state-of-the-art compression rates and accuracy.

\subsection{Experiment Setup} 

Our standard Shapley Context Pruning (SCP) model uses a computationally lightweight Deep Sets architecture with a hidden dimension of 1024 and a latent dimension of 512, totalling only 3.03M parameters (do not contain embedding parameters, by default the total parameters in the pipeline are 25.3M). The default embedding model for the SCP model is MiniLM \cite{wang2020minilmdeepselfattentiondistillation}. To be specific, for baseline "Cross Encoder", we use \texttt{ms-marco-MiniLM-L6-v2} \cite{hf2021msmarcominilm} as Cross Encoder, which is different from our embedder \texttt{all-MiniLM-L6-v2}. We train the network on the MS MARCO \cite{bajaj2016ms}, HotpotQA \cite{yang2018hotpotqa}, 2WikiMultiHop \cite{ho2020constructing}, MuSiQue \cite{trivedi2022musique} and FEVER \cite{thorne2018fever} datasets for 10 epochs using an AdamW \cite{loshchilov2019decoupledweightdecayregularization} optimizer (learning rate 1e-3, cosine schedule). By default, models are trained on the corresponding dataset (80\% for training, 10\% for validation, 10\% for testing). The pairwise margin $\epsilon$ is set to 0.15. Compression rate is computed by token count, while keep rate is computed by the number of sentences; therefore, they are not exactly the same. The Pairwise \texttt{AUC}(Area Under Curve) metric, or Pairwise Ranking Accuracy, evaluates the performance of ranking the supporting sentences before the distractors. \texttt{R@0.X} metric represents the percentage of supporting sentences preserved, i.e., $\frac{\text{Preserved Supporting Sentences}}{\text{Total Supporting Sentences}}$. \texttt{EM} metric is short for "Exact Match", computed by strict string matching of the LLM-output and true answers. \texttt{F1} metric is the harmonic mean of the recall and precision. The datasets, baseline methods (and the exact version we use), and inference LLMs are listed in Table~\ref{tab:setup}. The LLM QA pipeline is as follows: the context is embedded at the sentence level, the sentences are reranked and pruned, then fed to the LLM for inference. After that, we apply an EM/F1 evaluation script or use an LLM-as-judge to output a performance score.

\subsection{Main Result}

Figure~\ref{fig:main_results} summarizes the empirical evaluation of SCP across four dimensions: reranking quality, downstream LLM QA accuracy, cross-model robustness, and long-context needle recall.

\begin{figure*}[t]
\centering
\includegraphics[width=1\textwidth]{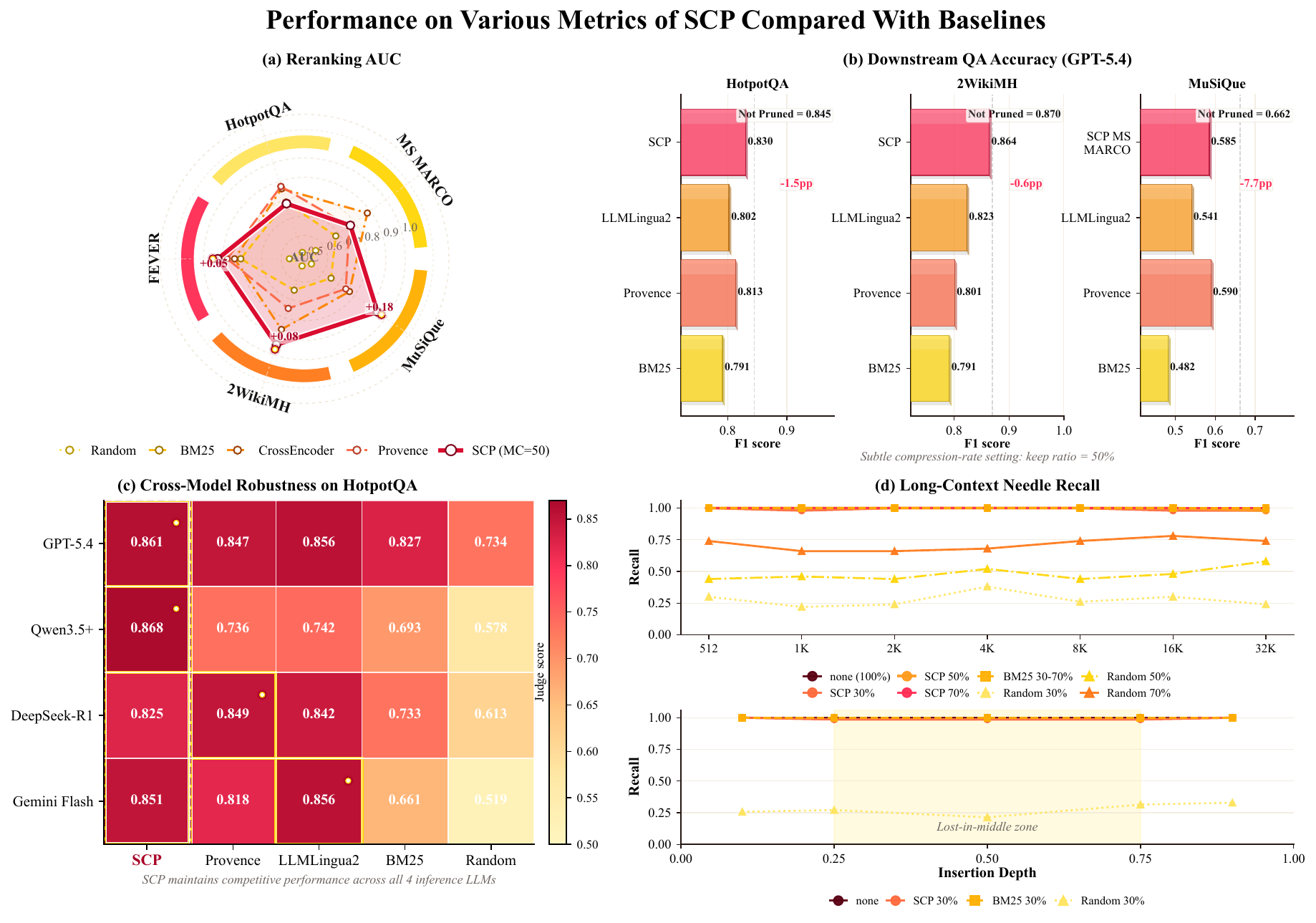}
\caption{Overview of SCP empirical results. (a)~Reranking AUC (Pairwise Ranking Accuracy) comparison; SCP achieves the highest AUC on MuSiQue, 2WikiMH, and FEVER. (b)~Downstream LLM QA F1 under 50\% compression. (c)~Cross-model judge-score heatmap on HotpotQA; highest score per model is bolded and boxed. (d)~NIAH needle-recall curves across context lengths (top) and insertion depths (bottom). BM25 dominates fact recall in long contexts since the "needle" typically contains query keywords, a scenario well-suited to BM25's sparse retrieval mechanism.}
\label{fig:main_results}
\end{figure*}

The original main result is shown in Table~\ref{tab:auc_master}, Table~\ref{tab:llm_qa_master}, Table~\ref{tab:hotpotqa_mmmp}, Table~\ref{tab:niah_length} and Table~\ref{tab:niah_depth}. We report the AUC of the reranking stage and the downstream QA performance of the LLMs. The AUC metric directly reflects the quality of the reranking stage, while the downstream QA performance reflects the end-to-end effectiveness of the pruning method.

For the MuSiQue rows of Table~\ref{tab:llm_qa_master}, we test four SCP variants: trained at paragraph level and pruning at paragraph level (\textit{paragraph*}); trained at sentence level and pruning at sentence level (\textit{sentlevel split}); trained at sentence level and pruning at paragraph level (\textit{sentlevel no split}); and trained on MS MARCO and pruning at sentence level (\textit{MS MARCO, cross}). We observe a misalignment between AUC and downstream LLM QA accuracy on MuSiQue for in-domain trained models, which is caused by the coarse-grained annotation from the dataset itself. The Cross Encoder used here is a variant of MiniLM. See Section~\ref{AUC_Just} for a detailed justification.

We also conduct experiments on CL-Bench and RULER (a benchmark for Needle-In-A-Haystack experiment, NIAH) to test the performance of SCP on long-context tasks. See Tables~\ref{tab:clbench_rubric}, \ref{tab:clbench_error_type_bm25_scp}, \ref{tab:clbench_ablation}, \ref{tab:clbench_overall}, \ref{tab:clbench_rubric_bm25_scp} and Figure~\ref{fig:clbench_overall} for CL-Bench and Tables~\ref{tab:niah_depth}, \ref{tab:niah_key_retrieval}, \ref{tab:niah_length}, \ref{tab:niah_mk5_loo_scp} for NIAH. The result shows that SCP outperforms BM25 on CL-Bench and equal to BM25 on RULER, showing the effectiveness of our approach in long context fact recalling and key sentence preservation. We provide the detail results and analysis in the appendix. We also provide a case study in real world in Tables~\ref{tab:casestudy_neuro_setup}, \ref{tab:casestudy_neuro_analysis}, \ref{tab:casestudy_neuro_acq}, \ref{tab:neuro_scp_diagnosis}.

The main results show that SCP reaches competitive performance with fewer parameters (even when accounting for the embedder parameters), and outperforms baselines on specific LLMs and datasets. These experiments validate our method as an engineering contribution; further improvements to fully realize SCP's potential in real-world tasks are left to future work.

\subsection{Ablations}

Figure~\ref{fig:ablation_results} presents the ablation studies, covering embedding quality robustness, Monte-Carlo sampling budget, Shapley versus LOO attribution, and a concrete redundancy failure mode.

\begin{figure*}[t]
\centering
\includegraphics[width=1\textwidth]{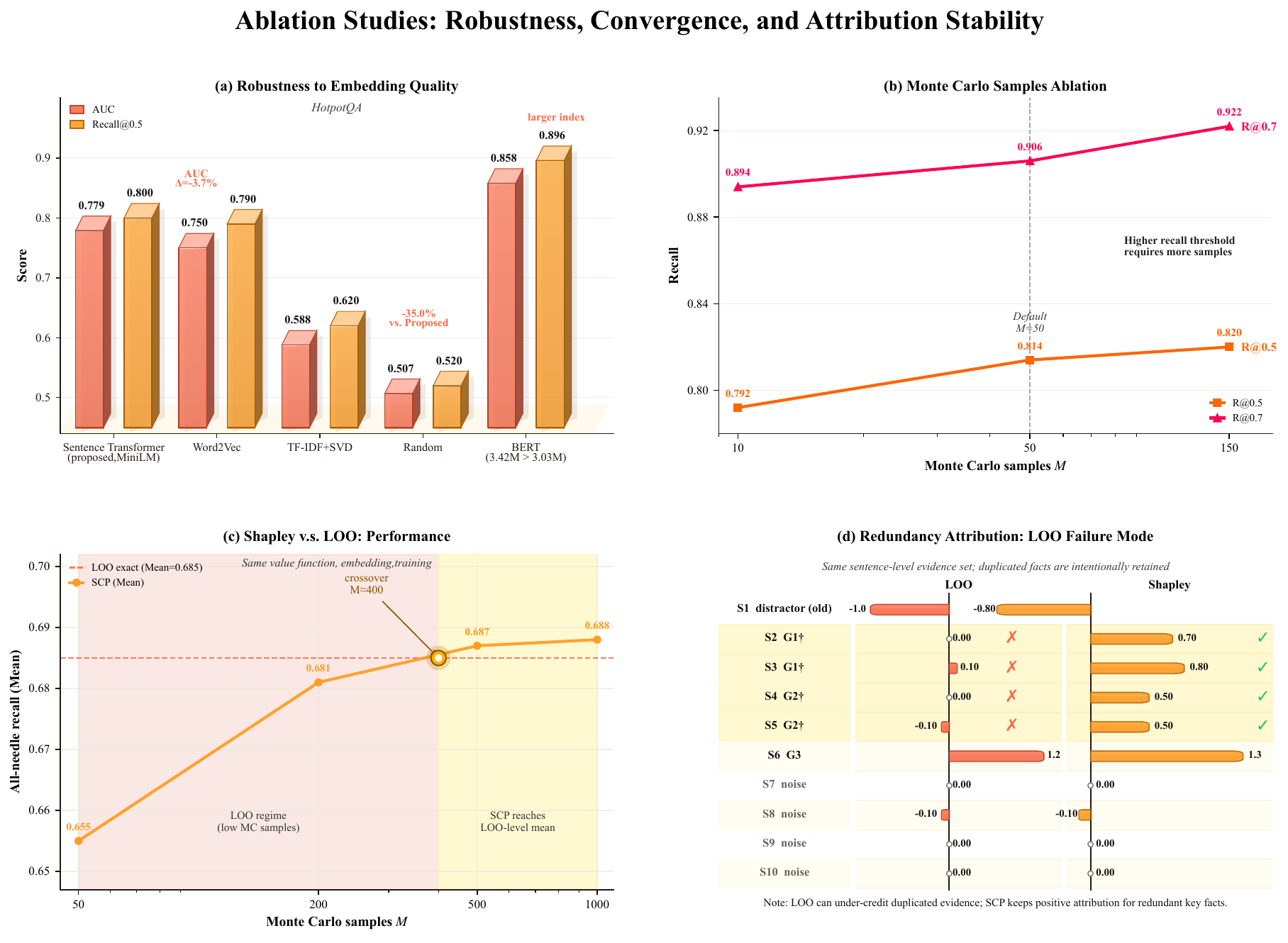}
\caption{Ablation studies. (a)~Embedding mechanism comparison; the default SentenceTransformer and BERT outperform shallow embeddings. SCP performs robustness under lower parameters embedder, i.e., MiniLM. Model Size: MiniLM(22M), BERT(110M). The model size of SCP will change slightly for embedding dimension alignment, Word2Vec(300dim), MiniLM(384dim), BERT(768dim). The total parameters of the value function are (3M $\pm$ 0.5M) in this experiment. (b)~Recall improves with more Monte-Carlo samples; the default $M=50$ is marked. (c)~SCP mean recall approaches and crosses LOO exact mean around $M \approx 400$. (d)~LOO under-credits duplicated supporting facts, whereas SCP preserves positive attribution.}
\label{fig:ablation_results}
\end{figure*}

\subsubsection{Embedders} 

A natural concern is whether SCP's effectiveness stems merely from the pre-trained embedding model that explicitly encodes queries. To rigorously decouple the source of SCP's robustness, we ablate the initial sentence-embedding input mechanism. We implement experiments on MiniLM, \texttt{bert-base-uncased} \cite{devlin2019bertpretrainingdeepbidirectional}, \texttt{word2vec-google-news-300} \cite{mikolov2013distributedrepresentationswordsphrases} and TF-IDF \cite{10.1016/0306-45738890021-0}. Table~\ref{tab:ablation_embed} shows the results. The results show that SCP is robust across modern embedding models. Though embedding quality does impact performance as expected, the similar AUC performance on MiniLM and Word2Vec shows that embedding dimension rather than embedding mechanism is the factor that determines the performance of SCP. See Section~\ref{sec:feat-ext} for more analysis.

\subsubsection{Attribution Strategies Comparison}

In this ablation, we compare the performance of SCP using the Shapley value attribution mechanism against a simpler Leave-One-Out (LOO) attribution method. The LOO method evaluates the marginal contribution by simply removing one element from the full coalition, without averaging over all permutations. We shows that LOO strategy is reasonable choice for low computational budget, but not optimal for higher accuracy requirements and robust performance in Table~\ref{tab:niah_mk5_loo_scp}, \ref{tab:loo_vs_shapley_updated}. See Section~\ref{sec:LOO} for the full analysis and Table~\ref{tab:redatk_g3} for a concrete example of the redundancy failure mode of LOO.

\subsubsection{Monte-Carlo Sampling Number}

Table~\ref{tab:msmarco_sample}, \ref{tab:niah_mk5_loo_scp} shows that SCP's accuracy improves as the Monte-Carlo sampling count increases. Table~\ref{tab:msmarco_sample} indicates that SCP is appropriate as a coarse-grained reranker at higher keep ratios, where it filters most low-value noise and lets a finer-grained reranker/pruner handle the preliminary processed context. However, it is worth noting that we do not take the embedding time into consideration.

\section{Discussion and Future Work}

\subsection{Limitations}

Although the SCP framework shows competitive performance in our validation experiments and provides a promising step toward a theoretical blueprint, several limitations remain to be addressed:

\textbf{Suboptimal Pruning Criteria:} In the current implementation, we predominantly rely on Top-K heuristics as the final pruning criterion; an adaptive threshold for SCP is not proposed in this paper.

\textbf{Estimation Error and High Variance in Different Scenarios:} Although the framework's error is bounded by Theorem~\ref{thm:error-bound}, an appropriate sampling count for different scenarios has not yet been characterized in this work. We did not implement multiple runs in LLM QA tasks due to API budget constraints. However, experiments in different scenarios have shown the feasibility of SCP framework.

\textbf{The Gap between theoretical structure and engineering performance:} We introduce the concept ``Landscape of Context'', which is a theoretical structure for modeling the context. See Figure~\ref{fig:landscape_of_context} for an illustration. However, such structures needs supermodular value functions. As an initial step toward closing this gap, we 
propose \texttt{ConvexDeepSetsV2} in Appendix~\ref{sec:convexity} and Table~\ref{tab:convexv2_comparison}. 
to enforce input convexity in the value function. However, fully 
reconciling theoretical guarantees and empirical performance 
remains an open question.

\subsection{Future Work}

Our study bridges cooperative game theory with context reranking, providing several promising direction for future research:

\begin{itemize}
    \item \textbf{Advanced Interaction-Aware Pruning Algorithms:} Addressing the limitations outlined above, future work could explore self-adapting thresholding pruning algorithms that may dynamically change $K$ under different scenarios. Furthermore, theoretical insights in Section~\ref{sec:shap_all} may provide structural priors for designing more efficient pruning algorithms.
    \item \textbf{Theoretical Structure implications:} The cooperative game formulation admits a rich analytical toolkit: "core" provides structural insights into stable context configurations, the supermodularity suggests a hierarchical multi-granularity pruning strategy (Appendix~\ref{sec:shap_all}). Future work could empirically validate how these structures manifest in real-world retrieval scenarios.See more in Section~\ref{sec:more_future}.
    \item \textbf{Extending Beyond Context Engineering:} The SCP framework's generality suggests applications beyond context pruning, such as feature selection in software test-case prioritization. Exploring these extensions could further demonstrate the versatility of cooperative game-theoretic approaches.
\end{itemize}

\section{Conclusion}

This paper introduces Shapley Context Pruning (SCP), framing context reranking as a cooperative game. A lightweight Deep Sets network with pairwise ranking and Monte-Carlo sampling makes Shapley attribution scalable. Empirically, SCP outperforms baselines on multi-hop benchmarks while remaining interpretable via approximately axiomatic attribution. Beyond, this paper provides a conceptual guideline towards future work in context engineering. These results bridge cooperative game theory with practical context engineering for scalable, interpretable long-context processing.

\bibliographystyle{plainnat}
\setcitestyle{numbers}
\bibliography{references}

\newpage
\appendix

\section{Towards a Theoretical Framework: A Cooperative Game-Theoretic Perspective}\label{sec:shap_all} 

In this section, we outline a promising conceptual blueprint beyond the empirical experimental validation of this paper. Our primary goal here is to inspire future researchers by demonstrating how the rich mathematical machinery of cooperative game theory can provide a theoretically grounded perspective for understanding the hidden structure of context. While the empirical sections of this paper operationalize these concepts using Deep Sets and neural approximations, this appendix serves as a theoretical sandbox. We aim to show how concepts like the "Core" and "Convex Games" can reframe the way the NLP community views context synergy, offering the formal intuition beneath the "Landscape of Context", i.e., how context looks like under certain value function. Before that, we first provide a more detailed comparison of different Shapley-based attribution frameworks in Table~\ref{tab:diff_shapley}, to clarify the unique position of SCP in the broader landscape of Shapley applications.

\begin{table}[h]
\centering
\caption{Comparison of Shapley-Based Attribution Frameworks}
\label{tab:diff_shapley}
\begin{tabular}{llccc}
\toprule
\cellcolor{yellow!15} & \cellcolor{yellow!15}\textbf{$v(S)$ type} & \cellcolor{yellow!15}\textbf{Granularity} & \cellcolor{yellow!15}\textbf{Scalable} & \cellcolor{yellow!15}\textbf{Learns} \\
\cellcolor{yellow!15} & \cellcolor{yellow!15} & \cellcolor{yellow!15}\textbf{("players")}  & \cellcolor{yellow!15}\textbf{($v(S)$)?} & \cellcolor{yellow!15}\textbf{$v(S)$?} \\
\midrule
Data Shapley(Original Work) \cite{ghorbani2019datashapley} & Model perf. & Data sample  & \redxmark & \redxmark \\
SHAP \cite{lundberg2017shap} & Surrogate model & Feature  & \greencmark & \redxmark \\
TokenSHAP \cite{goldshmidt2024tokenshapinterpretinglargelanguage} & Hard-coded (TF- & Token  & \greencmark & \redxmark \\
& IDF/cosine) & & & \\
\textbf{SCP (Ours)} & \textbf{Deep Sets} & \textbf{Sentence}  & \greencmark & \greencmark \\
\bottomrule
\end{tabular}
\end{table}

Table~\ref{tab:diff_shapley} summarizes the comparison of different Shapley-based attribution frameworks that answer completely different questions. Notably, SCP is the only trainable, scalable framework that directly attributes at the context unit level, making it uniquely suited for context pruning applications. The other frameworks fall into explainable AI (XAI) paradigms, where the Shapley value is used to attribute feature importance for a fixed model, rather than to rank discrete context units for pruning. SCP's direct optimization of a value function that maps subsets of context to performance metrics allows it to leverage the full theoretical properties of Shapley values in a practical, scalable manner.

It is worthy of notice that we have not designed any specific algorithm or training paradigm using the following concepts, since any validation of these algorithms and paradigms would require a separate paper. The following discussion is purely conceptual and theoretical, but we hope that these discussions can inspire future research in this direction.

\subsection{Formulating Context as a Cooperative Game}\label{sec:shap}

Instead of viewing a long prompt as a flat sequence of tokens, we model the context as a grand coalition of semantic components (e.g., sentences or paragraphs). The objective of the LLM reading this context can be abstracted as a value function.

\begin{definition}[Context Value Game] 
Let $N = \{c_1, c_2, \ldots, c_n\}$ be the set of context components. We define a fundamental value function $v: 2^N \rightarrow \mathbb{R}$ with $v(\emptyset) = 0$, where $v(S)$ represents the utility or information completeness (e.g., the likelihood of generating the correct answer) provided by a subset of components $S \subseteq N$. The pair $(N, v)$ constitutes a cooperative game.
\end{definition}

In natural language, sentences rarely contribute value independently. Information is riddled with logical prerequisites, anaphora, and multi-hop reasoning chains. Therefore, the marginal contribution of a component $c_i$, defined as $v(S \cup \{c_i\}) - v(S)$, is highly dependent on what information $S$ is already present. 

To assign a fair "importance score" to each component without bias toward any specific reading order or subset, we compute the expected marginal contribution over all possible information-accumulation sequences, which is precisely the Shapley value:

\begin{definition}[Shapley Importance Score]
The attribution score $\phi_i(v)$ for component $c_i$ in context $N$ under value function $v$ is the Shapley value:
\[
\phi_i(v) = \sum_{S \subseteq N \setminus \{c_i\}} \frac{|S|!(|N| - |S| - 1)!}{|N|!} [v(S \cup \{c_i\}) - v(S)]
\]
\end{definition}

The score obeys the four axiomatic properties:
\begin{itemize}
    \item Efficiency: The sum of the Shapley values of all players equals the total value of the grand coalition, i.e., \(\sum_{i \in N} \phi_i(v) = v(N)\).
    \item Symmetry: If two players \(i\) and \(j\) contribute equally to all coalitions, then they receive the same Shapley value, i.e., if \(v(S \cup \{i\}) = v(S \cup \{j\})\) for all \(S \subseteq N \setminus \{i, j\}\), then \(\phi_i(v) = \phi_j(v)\).
    \item Dummy Player: If a player \(i\) does not contribute to any coalition (i.e., \(v(S \cup \{i\}) = v(S)\) for all \(S \subseteq N \setminus \{i\}\)), then their Shapley value is zero, i.e., \(\phi_i(v) = 0\).
    \item Additivity: For two games with characteristic functions \(v\) and \(w \), the Shapley value of the combined game \(v + w\) is the sum of the Shapley values of the individual games, i.e., \(\phi_i(v + w) = \phi_i(v) + \phi_i(w)\) for all \(i \in N\).
\end{itemize}

These four axioms are not merely mathematical conveniences; they correspond to minimal desiderata that any reasonable context-reranking mechanism should satisfy:
\begin{itemize}
    \item \textbf{Efficiency} guarantees \textit{information completeness}: the sum of all sentence scores equals the total value of the full context, so no utility is left unexplained or double-counted.
    \item \textbf{Symmetry} enforces \textit{functional equivalence}: if two sentences provide identical marginal contributions across all subsets, the reranker must treat them identically---this naturally handles semantically duplicated retrievals.
    \item \textbf{Dummy Player} ensures \textit{noise suppression}: any sentence that never changes the answer quality (a pure distractor) receives zero score and is pruned first.
    \item \textbf{Additivity} enables \textit{multi-objective decomposition}: if the total value is the sum of independent sub-tasks (e.g., correctness + faithfulness), the attribution of each sub-task can be computed separately and then linearly combined without retraining the value function.
\end{itemize}

Because these four properties are natural requirements for a fair and interpretable reranking system, the following classical result acquires a contextual justification: the Shapley value is the \emph{only} attribution mechanism compatible with all of them.

\begin{theorem}[Axiomatic Characterisation of Context Attribution \cite{shapley1953value}]
Let $\mathcal{G}$ be the set of all cooperative games on a fixed player set $N$. The Shapley value $\phi: \mathcal{G} \to \mathbb{R}^{n}$ is the unique attribution mechanism that simultaneously satisfies Efficiency, Symmetry, Dummy Player, and Additivity.
\end{theorem}

\subsection{Synergy, Stability, and the Core}

If we simply truncate the context by selecting the top-$K$ sentences with the highest Shapley values, how do we know the resulting subset is "stable"? In cooperative game theory, stability means that no unselected group of players can form a coalition that yields a higher value than their current assigned payoff. This concept is encapsulated by the \textit{Core}.

\begin{definition}[The Contextual Core] 
Given a context game $(N, v)$, the Core $C(v)$ is the set of all payoff allocations $x \in \mathbb{R}^N$ that efficiently distribute the total context value ($ \sum_{i \in N} x_i = v(N) $) while ensuring no subset is "undervalued" relative to its intrinsic synergy:
\[
C(v) = \left\{x\in \mathbb{R}^N \;|\; \sum_{i \in S} x_i \geq v(S) \quad \text{for all } S \subseteq N\right\}
\]
\end{definition}

In applied context engineering, the core ensures that if a subset of sentences $S$ contains a highly synergistic logical chain (e.g., a multi-hop bridge), the individual scores $x_i$ assigned to those sentences must sum to at least $v(S)$. If the core is empty, it means the context contains conflicting synergies that cannot be linearly decomposed.

We call a subset $S$ \textbf{tight} at a core allocation $x$ when its scores sum to \emph{exactly} its value, i.e., $\sum_{i \in S} x_i = v(S)$. In context-pruning terms, this means the importance scores of the sentences in $S$ add up to precisely the worth of $S$ as a whole---no surplus, no deficit. A tight set is therefore a ``natural information block'': its internal attribution is fully explained by its member scores, suggesting a coherent, self-contained unit of meaning. We will see that these tight sets form the nodes of the context hierarchy.

To accommodate the noisy nature of LLM generation, we can relax this to the \textit{$\epsilon$-Core}:
\begin{definition}[$\epsilon$-Core]
\[
C_\epsilon(v) = \left\{x\in \mathbb{R}^N \;|\; \sum_{i \in N} x_i = v(N), \sum_{i \in S} x_i \geq v(S) - \epsilon \quad \text{for all } S \subseteq N\right\}
\]
\end{definition}
A non-empty $\epsilon$-core suggests that the attribution mechanism can be stable under bounded noise, providing theoretical motivation for using convex approximations in our neural network design.

\subsection{Convex Games and the Context Hierarchy}

To deeply understand the structural topology of information, we study a specific class of games where the value function exhibits supermodularity (often termed "convex games" in cooperative game theory). For non-convex scenario, we provide a possible idea, see paragraph~\ref{para:solving_non_convex}.

\begin{definition}[Convex Context Game (Supermodularity)]
A context game $(N, v)$ is convex if its value function $v$ is supermodular: for any two subsets $S, T \subseteq N$,
\[
v(S) + v(T) \leq v(S \cup T) + v(S \cap T)
\]
\end{definition}
In a convex game, a "snowball effect" can occur: the marginal utility of a sentence is non-decreasing as the surrounding context grows. This provides a natural abstraction for multi-hop reasoning, where a bridge entity may be uninformative alone but valuable when joined with its paired premise.

A foundational property of a convex game is that its Core $C(v)$ is large and non-empty. The extreme points of this Core induce structured families of subsets called \textit{Laminar Families}.

\begin{definition}[Laminar Family]
A family of sets $\mathcal{F}$ is called laminar if for any two sets $A, B \in \mathcal{F}$, either $A \subseteq B$, $B \subseteq A$, or $A \cap B = \emptyset$.
\end{definition}

\begin{figure}
\centering
\includegraphics[width=1\textwidth]{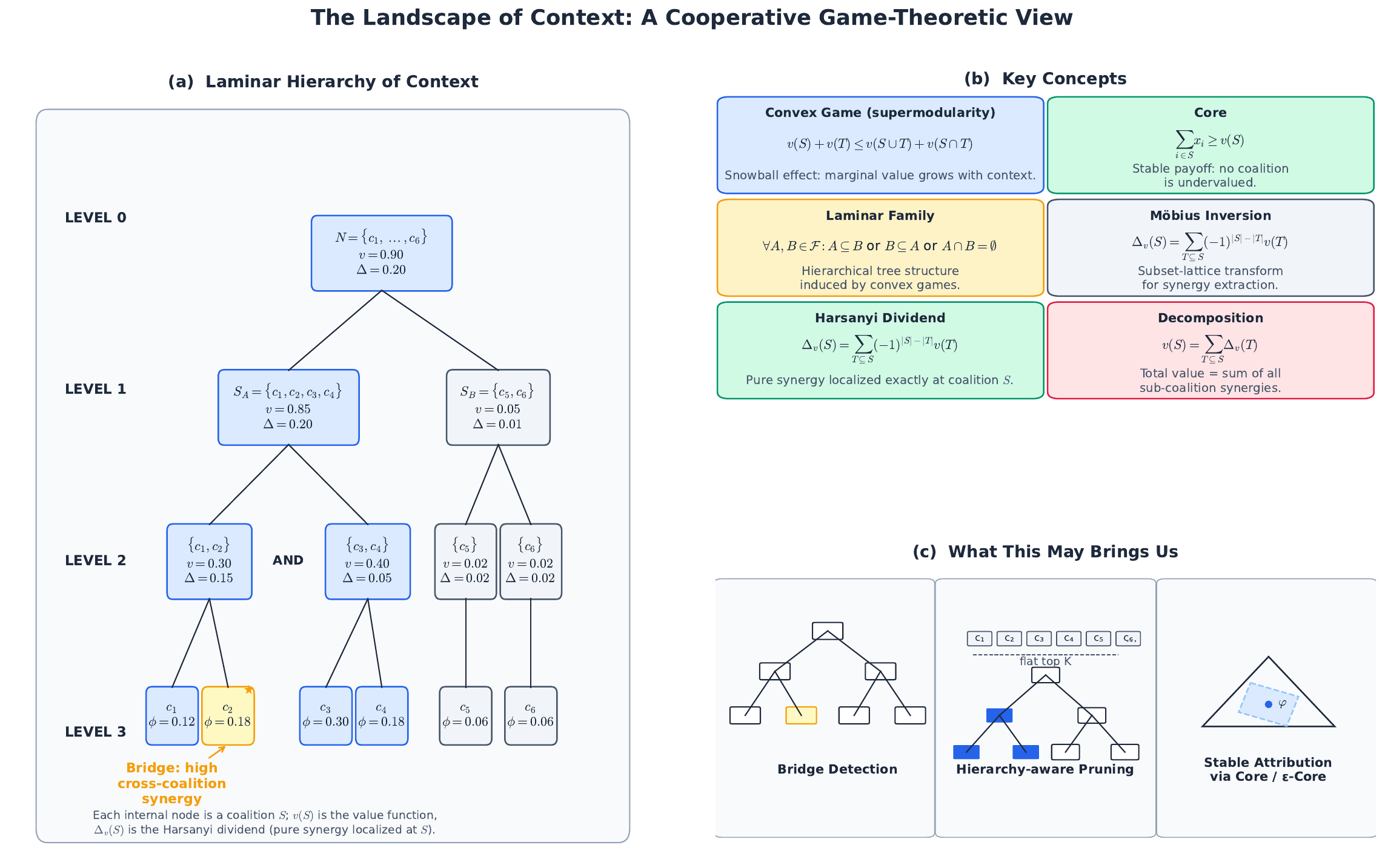}
\caption{Landscape of Context: What Will This Look Like? In this figure, we show an example of a context hierarchy induced by a convex value function. The specific value function and the value in the figure is not important, since they will vary across different algorithms and scenarios. The key point is that the tree structure itself is a direct consequence of the convexity of the value function, which ensures that the tight sets form a laminar family.}
\label{fig:landscape_of_context}
\end{figure}

Every extreme point $x$ in the core of a convex game defines a specific nested sequence of tight subsets $T_v(x) = \{S \subseteq N \mid \sum_{i \in S} x_i = v(S)\}$. Because this set system is laminar, each parent set either remains a leaf or splits into mutually disjoint child subsets. Concretely, a set $S$ splits whenever there exist disjoint proper subsets $T_1, T_2 \subsetneq S$ that are also tight and together cover $S$, meaning $v(S)$ decomposes additively across those children with no residual cross-subset synergy. If no such decomposition exists, $S$ remains a leaf (typically an individual sentence or an irreducible cohesive group). This yields a hierarchical tree: the root is the entire context, intermediate nodes branch into non-overlapping topical paragraphs or logical chains, and the leaves are individual sentences (Figure~\ref{fig:landscape_of_context}).

\paragraph{From Context to Tree: A Concrete Procedure.} To see how a concrete context unfolds into a tree, fix a core allocation $x$ (e.g., the Shapley value). The construction proceeds in three steps. \textbf{(1)~Identify tight sets.} For every subset $S \subseteq N$, check whether $\sum_{i \in S} x_i = v(S)$. The collection $\mathcal{T}(x)$ of all such sets always contains $N$ (by efficiency) and all singletons whose standalone worth equals their allocated score. \textbf{(2)~Uncross.} If two tight sets $A, B \in \mathcal{T}(x)$ cross (i.e., $A \cap B \neq \emptyset$ but neither contains the other), replace them by $A \cup B$ and $A \cap B$, which are also tight by convexity. Since $A \cap B \subseteq A \cup B$, this pair is laminar. Repeating this yields a maximal laminar family $\mathcal{L} \subseteq \mathcal{T}(x)$. \textbf{(3)~Build the tree.} For each $S \in \mathcal{L}$, its children are the maximal proper subsets of $S$ that also belong to $\mathcal{L}$. These children are pairwise disjoint and cover $S$ whenever $S$ is not a leaf. The root is $N$; leaves are sets with no children. This tree reveals the latent hierarchical structure of the context: each node is a self-contained information block, and each split marks a boundary where the internal synergy is fully explained by the children.

If there are multiple tight sets that can split a parent, the resulting tree is not unique. This non-uniqueness reflects future design for hieararchy construction algorithms, which may select splits based on additional criteria (e.g., greedy strategies). 

Of course, not every node in this tree can be split further: a leaf has no tight proper subsets, meaning its internal synergy cannot be decomposed into smaller self-contained blocks. This irreducible synergy is not attributable to any single sentence, but arises from the \emph{cooperation} among the members of that leaf. To quantify exactly how much value is created by this cooperation itself, beyond what the sub-components already contribute: we apply M\"{o}bius inversion, which yields the \textit{Harsanyi Dividends}.

\begin{definition}[Harsanyi Dividends on Subsets \cite{Harsanyi1982}]
Given a convex context game $(N, v)$, the M\"{o}bius inversion of the value function yields the Harsanyi dividend $\Delta_v(S)$ for a coalition $S$:
\[
\Delta_v(S) = \sum_{T \subseteq S} (-1)^{|S|-|T|} v(T)
\]
\end{definition}

The Harsanyi dividend captures precisely the value generated \textit{only} when all members of $S$ are present simultaneously. Through this, any context value function can be orthogonally decomposed:
\[
v(S) = \sum_{T \subseteq S} \Delta_v(T)
\]

\paragraph{Rationale} In the theoretical analysis, we are interested in how much value is created by the cooperation itself, beyond the value derived from sub-coalitions. Harsanyi dividends provide a tool that may be applied to identify essential ``bridge'' sentences rather than only supporting facts. 

We highlight the theoretical discussion around laminar families as a foundation for future context hierarchy design and context management systems. Future work may develop clever decomposition strategies and efficient tight set identification methods to accelerate the construction of the landscape. See Section~\ref{sec:more_future} for more discussion on future work in this direction.

\subsection{Landscape of Context}\label{sec:landscape}

Building on the preceding concepts, the ``Landscape of Context'' may be defined as a forest of game trees, each induced by a convex value function. We can further explore how to design neural architectures that can effectively capture this structure while remaining computationally tractable. This motivates our proposal of Convex Deep Sets in the next section.

\begin{definition}[Landscape of Context]
The Landscape of Context is a forest induced by a family of convex functions $\{v_i\}_{i = 1}^n$, where each tree corresponds to a convex value function defined over the same set of context components. Each node in the tree represents a coalition of sentences, and the edges represent the hierarchical relationships among these coalitions.
\end{definition}

We note that each value function $v_i$ does not determine a unique tree on its own: each extreme point of the core $C(v_i)$ induces a distinct laminar chain, with up to $n!$ such chains in total. Selecting a specific tree therefore requires an auxiliary rule---for example, a prior permutation, or a partition that maximizes Harsanyi dividends. We leave the design of such selection rules to future work.

\subsection{Convex Value Functions}\label{sec:convexity}

While standard Deepsets act as universal approximators, the resulting baseline value function typically lacks explicit analytical structure. Building upon the formalized framework in Appendix A, we emphasize that certain mathematical properties, particularly convexity and supermodularity, are useful when 
mapping a combinatorial hierarchy among contextual elements. To merge deep learning representations with this structural prior, we propose a specialized subset configuration termed \textit{Convex Deep Sets} utilizing convex neural networks.

\begin{definition}[Convex Deep Sets]
We define the Convex Deep Sets architecture as:
\[
f(X) = \rho_{convex}\left(\sum_{x\in X} \psi_{convex}(x)\right)
\]
where, in the context of this paper, the functionally constrained value distribution becomes:
\[
v(S;\theta) = \rho_{convex}\left(\sum_{c_i\in S} \psi_{convex}(c_i)\right)
\]
\end{definition}

By architecturally instilling this constraint, we introduce a convexity-based structural prior that is conceptually connected to supermodular (convex) games (discussed in Appendix A). While imposing convexity may restrict the hypothesis space, it can serve as an inductive bias for exploring hierarchical information aggregation in the proposed ``Landscape of Context.'' See Section~\ref{sec:conv-const} for implementation experiments.

\subsection{The error bound for Monte-Carlo Sampling} 

\begin{figure}
\centering
\includegraphics[width=1\textwidth]{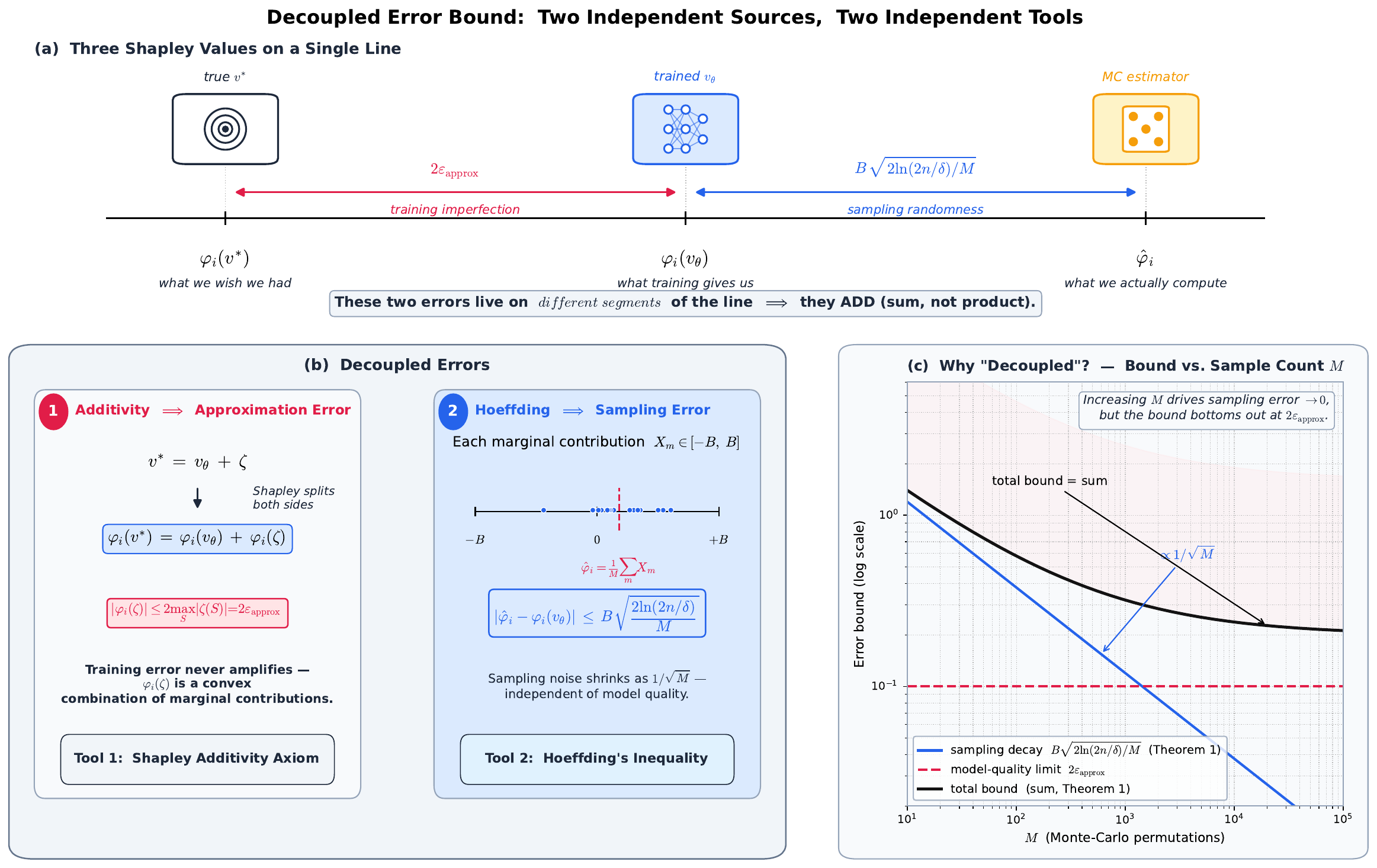}
\caption{The error bound for pruning attribution can be decoupled into two independent sources: the approximation error of the trained value function and the sampling error from Monte-Carlo estimation.}
\label{fig:error_bound}
\end{figure}

In our practical implementation, we use Monte-Carlo sampling to estimate the Shapley values. A natural question is: how does the estimation error of the Shapley values affect the downstream performance after pruning? We now define the trained value function with error and the corresponding Shapley value estimation, and then we can derive a theoretical error bound for the pruning performance based on the estimation error.

\begin{theorem}[Error Bound for Pruning Attribution (Recall)]
\label{thm:error-bound-appendix}
Let $v^*$ be the true value function and $v_{\theta}$ a trained approximation satisfying Assumption~\ref{ass:bounded}. Assume the point-wise approximation error is uniformly bounded by $\epsilon_{\mathrm{approx}}$, i.e.,
\[
|v^*(S) - v_{\theta}(S)| \leq \epsilon_{\mathrm{approx}}, \quad \forall S \subseteq N.
\]
Let $\hat{\phi}_i$ be the Monte-Carlo Shapley estimate of player $i$ obtained from $M$ independent uniform permutations (Algorithm~1). Then, with probability at least $1-\delta$ over the random permutations, the following holds simultaneously for all $i \in N$:
\[
\bigl|\hat{\phi}_i(v_{\theta}) - \phi_i(v^*)\bigr| \;\leq\; B\sqrt{\frac{2\ln(2n/\delta)}{M}} \;+\; 2\epsilon_{\mathrm{approx}}.
\]
\end{theorem}

\begin{proof}
We decompose the total error into two independent sources.

Define the residual $\zeta(S) = v^*(S) - v_{\theta}(S)$. By the additivity axiom of the Shapley value,
\[
\phi_i(v^*) = \phi_i(v_{\theta} + \zeta) = \phi_i(v_{\theta}) + \phi_i(\zeta).
\]
By the definition of the Shapley value, $\phi_i(\zeta)$ is a convex combination of the marginal contributions $\{\zeta(S\cup\{i\})-\zeta(S)\}_{S\subseteq N\setminus\{i\}}$. The uniform bound $|\zeta(S)| \leq \epsilon_{\mathrm{approx}}$ implies $|\zeta(S\cup\{i\})-\zeta(S)|\leq 2\epsilon_{\mathrm{approx}}$ for every $S$, and therefore
\[
\bigl|\phi_i(v_{\theta}) - \phi_i(v^*)\bigr| = \bigl|\phi_i(\zeta)\bigr| \leq 2\epsilon_{\mathrm{approx}}.
\]

Fix a player $i$. For each sampled permutation $\pi^{(m)}$ ($m=1,\dots,M$), let $S^{(m)}_{\pi,j}$ denote the coalition of elements preceding $i$ in $\pi^{(m)}$. The marginal contribution of $i$ along this permutation is
\[
X_m = v_{\theta}\bigl(S^{(m)}_{\pi,j} \cup \{i\}\bigr) - v_{\theta}\bigl(S^{(m)}_{\pi,j}\bigr).
\]
Because $v_{\theta}$ takes values in $[0,B]$ (Assumption~\ref{ass:bounded}), each $X_m$ lies in $[-B, B]$. The variables $\{X_m\}_{m=1}^{M}$ are i.i.d. with $\mathbb{E}[X_m] = \phi_i(v_{\theta})$. Consequently, the Monte-Carlo estimator $\hat{\phi}_i = \frac{1}{M}\sum_{m=1}^{M} X_m$ satisfies, by Hoeffding's inequality,
\[
\mathbb{P}\Bigl(\bigl|\hat{\phi}_i - \phi_i(v_{\theta})\bigr| \geq t\Bigr) \;\leq\; 2\exp\!\left(-\frac{M t^{2}}{2B^{2}}\right), \quad \forall t > 0.
\]
Choosing $t = B\sqrt{2\ln(2n/\delta)/M}$ yields
\[
\mathbb{P}\Bigl(\bigl|\hat{\phi}_i - \phi_i(v_{\theta})\bigr| \geq B\sqrt{\frac{2\ln(2n/\delta)}{M}}\Bigr) \;\leq\; \frac{\delta}{n}.
\]
Applying a union bound over all $n$ players gives the simultaneous high-probability guarantee.

Combining the two bounds,
\[
\bigl|\hat{\phi}_i - \phi_i(v^*)\bigr|
\;\leq\;
\underbrace{\bigl|\hat{\phi}_i - \phi_i(v_{\theta})\bigr|}_{\text{sampling error}}
\;+\;
\underbrace{\bigl|\phi_i(v_{\theta}) - \phi_i(v^*)\bigr|}_{\text{approximation error}}
\;\leq\;
B\sqrt{\frac{2\ln(2n/\delta)}{M}} + 2\epsilon_{\mathrm{approx}}.
\]
\end{proof}

\begin{theorem}[Sample Complexity for Top-K Ranking Preservation]
\label{thm:ranking-preservation}
Under the assumptions of Theorem~\ref{thm:error-bound-appendix}, let $\Delta_{(K)} = \phi_{(K)}(v^*) - \phi_{(K+1)}(v^*)$ denote the gap between the $K$-th and $(K+1)$-th true Shapley values. If $\Delta_{(K)} > 4\epsilon_{\mathrm{approx}}$, then with
\[
M \;\geq\; \frac{8B^{2}\,\ln(2n/\delta)}{\bigl(\Delta_{(K)} - 4\epsilon_{\mathrm{approx}}\bigr)^{2}}
\]
independent permutations, the estimated Top-$K$ set $\hat{S}_K = \{i : \hat{\phi}_i \geq \hat{\phi}_{(K)}\}$ equals the true Top-$K$ set $S_K^* = \{i : \phi_i(v^*) \geq \phi_{(K)}(v^*)\}$ with probability at least $1-\delta$.
\end{theorem}

\begin{proof}
By Theorem~\ref{thm:error-bound-appendix}, with probability at least $1-\delta$ we have $|\hat{\phi}_i - \phi_i(v^*)| \leq \eta$ simultaneously for all $i$, where $\eta = B\sqrt{2\ln(2n/\delta)/M} + 2\epsilon_{\mathrm{approx}}$. For any $i \in S_K^*$ and $j \notin S_K^*$, we have $\phi_i(v^*) \geq \phi_{(K)}(v^*) > \phi_j(v^*) + \Delta_{(K)}$. If $\eta < \Delta_{(K)}/2$, then
\[
\hat{\phi}_i \;\geq\; \phi_i(v^*) - \eta \;>\; \phi_{(K)}(v^*) - \frac{\Delta_{(K)}}{2} \;=\; \phi_{(K+1)}(v^*) + \frac{\Delta_{(K)}}{2} \;>\; \phi_j(v^*) + \eta \;\geq\; \hat{\phi}_j,
\]
so the ordering between $S_K^*$ and its complement is preserved. The condition $\eta < \Delta_{(K)}/2$ is equivalent to $B\sqrt{2\ln(2n/\delta)/M} < \Delta_{(K)}/2 - 2\epsilon_{\mathrm{approx}}$, which yields the stated bound on $M$ after rearrangement (requiring $\Delta_{(K)} > 4\epsilon_{\mathrm{approx}}$).
\end{proof}

\section{Alternative Pruning Strategies}

\subsection{Top-K Pruning} 

We provide the pruning algorithm in pseudo-code format in Algorithm~\ref{alg:shapley_pruning_topk}. The algorithm is simple but effective in terms of coarse-grained pruning. 

\begin{algorithm}
\captionsetup{labelformat=empty}
\caption{\colorbox{yellow!15}{\parbox{\dimexpr\linewidth-2\fboxsep\relax}{\textbf{Algorithm~\thealgorithm} Shapley Pruning Algorithm (Top-K Version)}}}
\label{alg:shapley_pruning_topk}
\begin{algorithmic}[1]
\Require A context set \(C = \{c_1, c_2, \ldots, c_n\}\), optimized value function \(v_\theta\), budget \(K\)
\Ensure A pruned context subset \(C' \subseteq C\)
\State Initialize an empty subset \(C' = \emptyset\)
\State Compute the Shapley attribution \(\phi_i(v_{\theta})\) for each \(c_i \in C\) with Monte Carlo sampling
\State Identify the index set \(I_{TopK}\) corresponding to the \(K\) highest \(\phi_i\) scores
\For{each part \(c_i \in C\)}
    \If{\(i \in I_{TopK}\)}
        \State \(C' \gets C' \cup \{c_i\}\)
    \EndIf
\EndFor
\State \textbf{return} \(C'\)
\end{algorithmic}
\end{algorithm}

\subsection{Average Pruning (Tried, but not as effective as Top-K)}

In an early version of this work, we tried a straightforward pruning strategy based on the average Shapley value as a threshold. The algorithm is shown in Algorithm~\ref{alg:shapley_pruning_average}. We found that this simple strategy could not meet stable compression requirements, so we switched to the Top-K version used in the main paper. We still include this version here as a reference for future research on more advanced pruning strategies. In future work, adaptive thresholding strategies may be applied to achieve better compression rates across diverse scenarios, but this is far beyond the scope of this paper, and we leave it as an open question for future research.

\begin{algorithm}
\captionsetup{labelformat=empty}
\caption{\colorbox{yellow!15}{\parbox{\dimexpr\linewidth-2\fboxsep\relax}{\textbf{Algorithm~\thealgorithm} Shapley Pruning Algorithm (Average-Threshold Version)}}}
\label{alg:shapley_pruning_average}
\begin{algorithmic}[1]
\Require A context set \(C = \{c_1, c_2, \ldots, c_n\}\), optimized value function \(v_\theta\)
\Ensure A pruned context subset \(C' \subseteq C\)
\State Initialize an empty subset \(C' = \emptyset\)
\State Compute the Shapley attribution \(\phi_i(v_{\theta})\) for each \(c_i \in C\)
\State Calculate threshold $\bar{\phi} \gets \max(0, \frac{1}{n} \sum_{i=1}^{n} \phi_i)$
\For{each part \(c_i \in C\)}
    \If{\(\phi_i \geq \bar{\phi}\)}
        \State \(C' \gets C' \cup \{c_i\}\)
    \EndIf
\EndFor
\State \textbf{return} \(C'\)
\end{algorithmic}
\end{algorithm}

\section{Experiment Details} 

This appendix provides comprehensive experimental data that complements the main results presented in Section 4. We include detailed performance metrics across all datasets, ablation studies, and cross-domain generalization analysis. "sentlevel" means that SCP is trained on sentence-level annotations, while "paragraph" means that SCP is trained on paragraph-level annotations. "split" means that the model is able to prune at the sentence level, while "no split" means that the model can only prune at the paragraph level.

\subsection{Experiment Setup Details}

During inference, we apply our parallel Monte-Carlo estimator batched on an Nvidia A800 GPU. We use the SCP model trained on MS MARCO by default for inference; the use of models trained on other datasets is specified before each test case. Inference time for each model is measured as the wall-clock cost of the pruning process, including the forward passes of the value function and Monte-Carlo sampling. Specifically, the reported inference time excludes sentence-transformer encoding time, since we assume that all sentences have been encoded before the pruning process (i.e., a static environment) and the encoding time is the same for all methods. We report the average inference time per sample across the evaluated datasets.

\begin{table}[t]
\centering
\small
\caption{Overview of datasets, baselines, and inference LLMs used in our experiments.}
\label{tab:setup}
\begin{tabular}{lp{9.5cm}}
\toprule
\cellcolor{yellow!15}\textbf{Name} & \cellcolor{yellow!15}\textbf{Description} \\
\midrule
\multicolumn{2}{l}{\textit{Datasets}} \\
\midrule
MS MARCO~\cite{bajaj2016ms} & Large-scale passage retrieval (v2.1). Dev set: 25,147 samples. \\
HotpotQA~\cite{yang2018hotpotqa} & Multi-hop QA. Customized test split: 612 samples (cross-domain zero-shot). \\
MuSiQue~\cite{trivedi2022musique} & Compositional multi-hop questions (2--4 hops). 997 samples. \\
2WikiMultiHopQA~\cite{ho2020constructing} & Wikipedia/Wikidata relational reasoning. 7,737 samples. \\
FEVER~\cite{thorne2018fever} & Fact verification and evidence retrieval. 3,848 test samples. \\
RULER~\cite{hsieh2024rulerwhatsrealcontext} & Needle-in-haystack long-context retrieval (Paul Graham Haystack). \\
CL-Bench~\cite{clbench2026} & Long-context benchmark for LLM-perceived context pruning. \\
\midrule
\multicolumn{2}{l}{\textit{Baselines}} \\
\midrule
Random & Random context selection (lower bound). \\
BM25~\cite{robertson2009bm25} & Classical term-based retrieval. \\
CrossEncoder~\cite{wang2020minilmdeepselfattentiondistillation} & \texttt{ms-marco-MiniLM-L-6-v2}, a cross-encoder adapted from the original MiniLM. \\
LLMLingua-2~\cite{pan2024llmlingua2} & Data distillation-based prompt compression. \\
Provence~\cite{chirkova2025provence} & We use \texttt{open-provence-reranker-large-v1}\cite{yuichi-tateno-2025-open-provence} for experiments, which reproduced the provence model exactly from the method provided in the original paper. Another version \texttt{provence-reranker-debertav3-v1} has 0.4B parameters, which is the model provided by the authors. This model is a reproductive model by the community. This choice is made with the concern that engineering optimizations in the original implementation may introduce additional analytical complexity but not pure method comparison.\\
\midrule
\multicolumn{2}{l}{\textit{Inference LLMs}} \\
\midrule
GPT-5.4 \cite{openai2026gpt54}& Closed-source LLM (OpenAI). \\
Gemini 3 Pro \cite{deepmind2026gemini3}& Closed-source LLM (Google). \\
Gemini 3 Flash \cite{deepmind2026gemini3}& Closed-source LLM (Google). \\
Qwen-3.5-Plus \cite{yang2025qwen3technicalreport} & Open-source LLM (Alibaba). \\
Deepseek-R1 \cite{Guo_2025} & Open-source LLM (DeepSeek). \\
\bottomrule
\end{tabular}
\end{table}

\subsubsection{Experimental Setup and Implementation Details}
To ensure reproducibility, we outline the precise model architecture, training configuration, and deployment hardware used in our experiments.

\paragraph{Architecture} We instantiate the value function using a Deep Sets neural network. The feature extractor $\psi$ is a symmetric 3-layer Multi-Layer Perceptron (MLP) mapping the input embedding dimension down to a latent dimension, while the regressor $\rho$ is a 3-layer MLP mapping the latent sum back to a scalar value. Specifically, we set the intermediate hidden dimension to 1024 and the latent dimension of the Deep Sets bottleneck to 512. Across all layers of both $\psi$ and $\rho$, we utilize Layer Normalization, ReLU activations, and a Dropout rate of $0.1$. The resulting model comprises exactly 3.03M parameters. Standard textual representations are injected into $\psi$ using SentenceTransformer variants (e.g., MiniLM, 384 dimensions) precomputed for training efficiency.

\paragraph{Training Paradigm} The network is optimized via the Pairwise Margin Ranking Loss with the margin set to $\epsilon = 0.15$. We employ the AdamW optimizer with a weight decay of $1\text{e-4}$ and a peak learning rate of $1\text{e-3}$. Learning rate scheduling utilizes cosine annealing, preceded by a warmup period equal to roughly $10\%$ of total epochs. Training runs for $10$ epochs across all variants. Training batches gather context states from datasets configured strictly to an $80\%/10\%/10\%$ train/validation/test split rule. 

\paragraph{Hardware} Due to the aggressive $3\text{M}$ parameter capacity and cached contextual representation paradigm, an entire 10-epoch training procedure converges uniformly in merely several hours on two NVIDIA A800 80GB GPUs. Furthermore, the memory overhead associated directly with the SCP framework is extremely low (averaging less than 4GB of active VRAM requirement outside standard embedding storage), indicating the need of the development of the better training algorithm. At inference time, as shown previously, evaluating $M=50$ Monte-Carlo samples operates highly efficiently, however, we recommend higher sampling number in practice for better performance, as it can be parallelized computed on GPU.

\paragraph{LLM API Version} All experiments with LLM usage are done within 27 March 2026 and 15 April 2026 with Official API. This information is mentioned for reproduction.  

\subsection{Main Results} 

Main result is shown in Table~\ref{tab:auc_master}, Table~\ref{tab:llm_qa_master} and Table~\ref{tab:hotpotqa_mmmp}. We report the AUC performance of all methods across all datasets in Table~\ref{tab:auc_master}, and the downstream LLM QA performance (EM and F1) in Table~\ref{tab:llm_qa_master}. We also report the multi-model, multi-pruner results on HotpotQA in Table~\ref{tab:hotpotqa_mmmp}.

\begin{table}[t]
\centering
\caption{Complete AUC performance comparison across all datasets (k=50\% compression). Best in bold.}
\label{tab:auc_master}
\begin{tabular}{lccccc}
\toprule
\cellcolor{yellow!15}\textbf{Method} & \cellcolor{yellow!15}\textbf{HotpotQA} & \cellcolor{yellow!15}\textbf{MS MARCO} & \cellcolor{yellow!15}\textbf{MuSiQue} & \cellcolor{yellow!15}\textbf{2WikiMH} & \cellcolor{yellow!15}\textbf{FEVER} \\
\midrule
Random & 0.514 & 0.552 & 0.525 & 0.519 & 0.557 \\
BM25 \cite{robertson2009bm25} & 0.787 & 0.680 & 0.650 & 0.650 & 0.808 \\
CrossEncoder \cite{nogueira2020passagererankingbert} & 0.860 & \textbf{0.881} & 0.767 & 0.863 & 0.840 \\
Provence \cite{chirkova2025provence} & \textbf{0.872} & 0.766 & 0.744 & 0.749 & 0.878 \\
\textbf{SCP (MC=50)} & 0.779 & 0.773 & \textbf{0.946} & \textbf{0.946} & \textbf{0.927} \\
\bottomrule
\end{tabular}
\end{table}

\begin{table*}[t]
\centering
\small
\caption{Downstream LLM QA results (GPT-5.4) across all datasets. Best pruning result per dataset in bold.}
\label{tab:llm_qa_master}
\begin{tabular}{llccc}
\toprule
\cellcolor{yellow!15}\textbf{Dataset} & \cellcolor{yellow!15}\textbf{Method} & \cellcolor{yellow!15}\textbf{EM} & \cellcolor{yellow!15}\textbf{F1} & \cellcolor{yellow!15}\textbf{Compress} \\
\midrule
\multirow{4}{*}{HotpotQA}
 & Not Pruned (100\%) & 0.791 & 0.845 & 1.000 \\
 & BM25 (50\%) & 0.734 & 0.791 & 0.548 \\
 & \textbf{SCP} (50\%) & \textbf{0.774} & \textbf{0.830} & 0.508 \\
 & Provence (50\%) & 0.756 & 0.813 & 0.488 \\
 & LLMLingua2 (50\%) & 0.744 & 0.802 & 0.465 \\
\midrule
\multirow{4}{*}{2WikiMH}
 & Not Pruned & 0.859 & 0.870 & 1.000 \\
 & BM25 (50\%) & 0.767 & 0.791 & 0.539 \\
 & \textbf{SCP} (50\%) & \textbf{0.852} & \textbf{0.864} & 0.540 \\
 & LLMLingua2 (50\%) & 0.804 & 0.823 & 0.465 \\
 & Provence (50\%) & 0.785 & 0.801 & 0.487 \\
\midrule
\multirow{5}{*}{MuSiQue}
 & Not Pruned (100\%) & 0.533 & 0.662 & 1.000 \\
 & BM25 (50\%) & 0.372 & 0.482 & 0.509 \\
 & \textbf{SCP MS MARCO} (50\%, cross) & \textbf{0.466} & 0.585 & 0.515 \\
 & Provence (50\%) & 0.463 & \textbf{0.590} & 0.485 \\
 & SCP sentlevel split (50\%) & 0.441 & 0.565 & 0.526 \\
 & SCP paragraph$^*$ (50\%) & 0.389 & 0.512 & 0.506 \\
 & SCP sentlevel (no split, 50\%) & 0.377 & 0.496 & 0.493 \\
\bottomrule
\end{tabular}
\end{table*}

\begin{table*}
\centering
\caption{LLM QA (Judge, GPT-5.4) across Multi-Model, Multi-Pruners on HotpotQA. Best per model in bold.}
\label{tab:hotpotqa_mmmp}
\begin{tabular}{lcccc}
\toprule
\cellcolor{yellow!15}\textbf{Method} & \cellcolor{yellow!15}\textbf{GPT-5.4} & \cellcolor{yellow!15}\textbf{Qwen3.5+} & \cellcolor{yellow!15}\textbf{DeepSeek-R1} & \cellcolor{yellow!15}\textbf{Gemini Flash} \\
\midrule
None        & 0.589 & 0.853 & 0.802 & 0.758 \\
Random      & 0.734 & 0.578 & 0.613 & 0.519 \\
BM25        & 0.827 & 0.693 & 0.733 & 0.661 \\
LLMLingua2  & 0.856 & 0.742 & 0.842 & \textbf{0.856} \\
Provence    & 0.847 & 0.736 & \textbf{0.849} & 0.818 \\
\textbf{SCP} & \textbf{0.861} & \textbf{0.868} & 0.825 & 0.851 \\
\bottomrule
\end{tabular}
\end{table*}

\subsection{Fact Recalling Supplementary Experiments} 

\subsubsection{Cross-Domain Experiment}

Tables~\ref{tab:msmarco_sample} show how a model trained on MS MARCO generalizes to unseen datasets.

\subsubsection{CL-Bench Results} 

\begin{figure}
\centering
\includegraphics[width=1\textwidth]{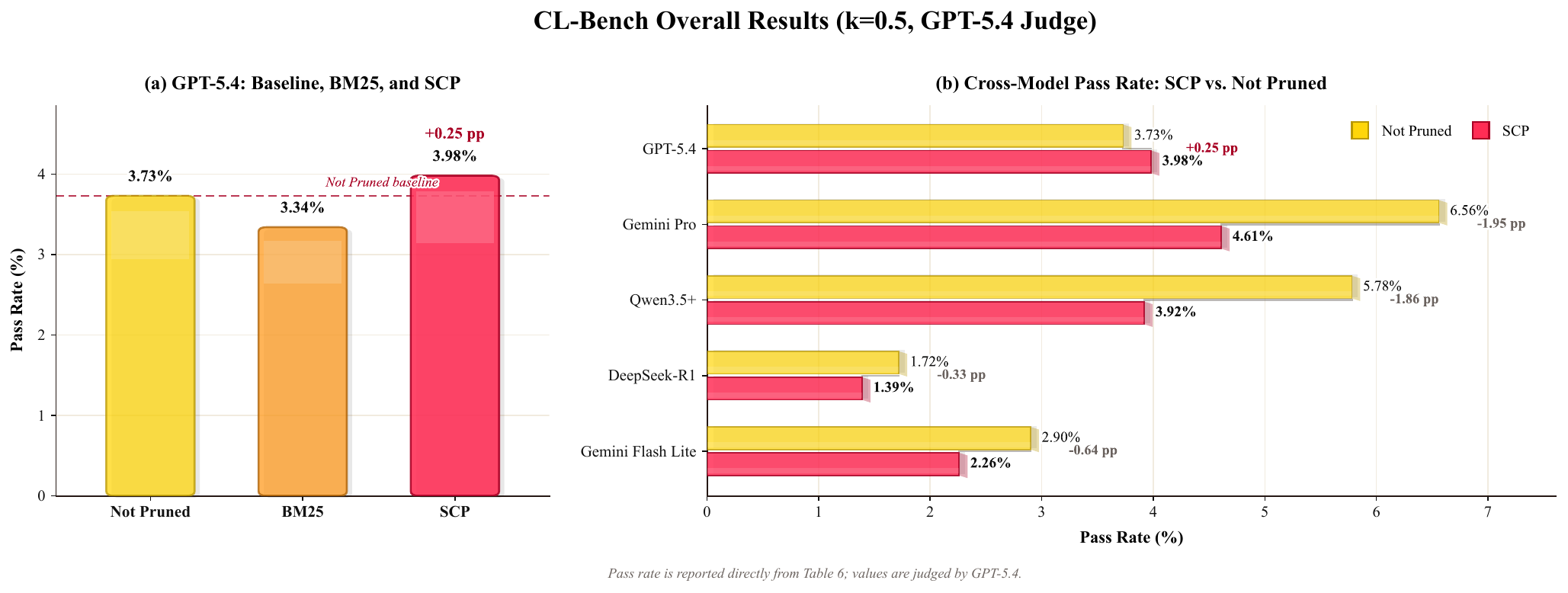}
\caption{CL-Bench overall results (k=0.5). Judged by GPT-5.4. GPT-5.4 performs not idealy on CL-Bench, mainly due to its format errors. We provide this result for supplementary reference, and we also report the ablation results with different judges in Table~\ref{tab:clbench_ablation}. SCP improve the performance of GPT-5.4 by 0.25\%, while BM25 hurts the performance by 0.39\%. However, this result is not statistically significant, and the performance of GPT-5.4 on CL-Bench is generally low, which may be due to the strict judging criteria of GPT-5.4. We include this result for completeness and reference, but we do not draw strong conclusions from it.}
\label{fig:clbench_overall}
\end{figure}

The main results for the CL-Bench experiment are shown in Table~\ref{tab:clbench_overall} and Figure~\ref{fig:clbench_overall}.

\begin{table}[h]
\centering
\caption{CL-Bench overall results (k=0.5). Judged by GPT-5.4.}
\label{tab:clbench_overall}
\begin{tabular}{lccl}
\toprule
\cellcolor{yellow!15}\textbf{Model} & \cellcolor{yellow!15}\textbf{Method} & \cellcolor{yellow!15}\textbf{Correct/Total} & \cellcolor{yellow!15}\textbf{Pass Rate} \\
\midrule
\multirow{3}{*}{GPT-5.4} & Baseline & 70/1875 & 3.73\% \\
 & BM25 & 63/1885 & 3.34\% ($-$0.39\%) \\
 & SCP & 74/1860 & 3.98\% (+0.25\%) \\
\midrule
\multirow{2}{*}{Gemini Pro} & Baseline & 124/1890 & 6.56\% \\
 & SCP & 80/1737 & 4.61\% ($-$1.95\%) \\
\midrule
\multirow{2}{*}{Qwen3.5+} & Baseline & 109/1885 & 5.78\% \\
 & SCP & 71/1810 & 3.92\% ($-$1.86\%) \\
\midrule
\multirow{2}{*}{DeepSeek-R1} & Baseline & 32/1858 & 1.72\% \\
 & SCP & 26/1873 & 1.39\% ($-$0.33\%) \\
\midrule
\multirow{2}{*}{Gemini Flash Lite} & Baseline & 55/1899 & 2.90\% \\
 & SCP & 43/1899 & 2.26\% ($-$0.64\%) \\
\bottomrule
\end{tabular}
\end{table}

Different from the official paper, we use GPT-5.4 as the judge, which is more strict than GPT-5.1 (Official used). See table \ref{tab:clbench_ablation}.

\begin{table}[h]
\centering
\caption{Performance of SCP on CL-Bench with Different Evaluating LLMs, Inference Model: GPT-5.4(not High), Not compressed}
\label{tab:clbench_ablation}
\begin{tabular}{lccccc}
\toprule
\cellcolor{yellow!15}\textbf{Evaluator} & \cellcolor{yellow!15}\textbf{Success Rate} & \cellcolor{yellow!15}\textbf{Format Error} & \cellcolor{yellow!15}\textbf{Context Ignore} & \cellcolor{yellow!15}\textbf{Context Misused} & \cellcolor{yellow!15}\textbf{Refusal} \\
\midrule
GPT-5.1 & 14.57\% & 84.1\% & 24.9\% & 23.7\% & 0.4\% \\
GPT-5.4 & 3.73\% & 73.7\% & 15.2\% & 7.6\% & 0.3\% \\
\end{tabular}
\end{table}

According to the official paper, the pair (GPT-5.2 High inference, GPT-5.1 judge) yields 18.1$\pm$0.8\% case
success, 33.9\% format error, 59.3\% context-ignore, 65.4\% context-misuse, and 2.4\% refusal. GPT-5.4
is more likely to produce format errors but less likely to produce context-ignore or context-misuse,
which suggests that GPT-5.4 enforces context content more strictly but is less strict about output
format. The behavior of the GPT series depends on the training objective set by OpenAI, which may change
across versions. We speculate that GPT-5.4 places stronger emphasis on factual accuracy and context
adherence, leading to more format errors when the model is forced to follow the context strictly,
whereas GPT-5.1 takes a more lenient approach to context adherence, leading to more context-ignore and
misuse cases.

Table~\ref{tab:clbench_rubric_bm25_scp} shows how SCP outperforms BM25---another popular coarse-grained retriever---on long-context QA, e.g.\ CL-Bench. We also report rubric-satisfaction bands. Detailed data comparing SCP and the unpruned context are given in Table~\ref{tab:clbench_rubric}.

\begin{table}[h]
\centering
\caption{CL-Bench requirement-satisfaction distribution (BM25 vs SCP, GPT-5.4 judge, $k=0.5$).}
\label{tab:clbench_rubric_bm25_scp}
\begin{tabular}{lccccc}
\toprule
\cellcolor{yellow!15}\textbf{Satisfaction} & \cellcolor{yellow!15}\textbf{BM25 n} & \cellcolor{yellow!15}\textbf{BM25 Pass} & \cellcolor{yellow!15}\textbf{SCP n} & \cellcolor{yellow!15}\textbf{SCP Pass} & \cellcolor{yellow!15}\textbf{Delta} \\
\midrule
0--20\% & 133 & 0.00\% & 142 & 0.00\% & +0.00\% \\
20--40\% & 281 & 0.00\% & 266 & 0.00\% & +0.00\% \\
40--60\% & 394 & 0.00\% & 375 & 0.00\% & +0.00\% \\
60--80\% & 568 & 0.00\% & 553 & 0.00\% & +0.00\% \\
80--100\% & 509 & 12.38\% & 524 & 14.12\% & +1.74\% \\
\bottomrule
\end{tabular}
\end{table}

\begin{table}[h]
\centering
\caption{CL-Bench results by rubric satisfaction bands.(GPT-5.4, k=0.5).}
\label{tab:clbench_rubric}
\begin{tabular}{lccccc}
\toprule
\cellcolor{yellow!15}\textbf{Satisfaction} & \cellcolor{yellow!15}\textbf{Baseline n} & \cellcolor{yellow!15}\textbf{Baseline Pass} & \cellcolor{yellow!15}\textbf{SCP n} & \cellcolor{yellow!15}\textbf{SCP Pass} & \cellcolor{yellow!15}\textbf{Delta} \\
\midrule
0--20\% & 153 & 0.00\% & 141 & 0.00\% & +0.00\% \\
20--40\% & 272 & 0.00\% & 264 & 0.00\% & +0.00\% \\
40--60\% & 381 & 0.00\% & 371 & 0.00\% & +0.00\% \\
60--80\% & 536 & 0.00\% & 543 & 0.00\% & +0.00\% \\
80--100\% & 529 & 13.23\% & 516 & 14.34\% & +1.11\% \\
\bottomrule
\end{tabular}
\end{table}

The ``Pass'' column reports how many tasks in each band were graded 100\% correct by the LLM. In the 0--80\% bands no task is fully correct, so the rate is automatically 0.

These tables show that SCP successfully reduces the number of answers in the low-satisfaction bands of the rubric and, as expected, occasionally prunes important information, which slightly lowers QA accuracy.

Table~\ref{tab:clbench_error_type_bm25_scp} reports the error-type distribution of the different methods.

\begin{table}[h]
\centering
\caption{Heuristic error-type ratios on failed CL-Bench samples (keyword tags over grading rationales, multi-label). Refusal is reported with a strict hard-refusal matcher; a broader inability-phrase rate is shown only as reference.}
\label{tab:clbench_error_type_bm25_scp}
\begin{tabular}{lcc}
\toprule
\cellcolor{yellow!15}\textbf{Error Type} & \cellcolor{yellow!15}\textbf{BM25 (n=1822)} & \cellcolor{yellow!15}\textbf{SCP (n=1820)} \\
\midrule
Format / structure errors & 59.17\% & 60.53\% \\
Missing key requirements & 52.85\% & 51.74\% \\
Context misuse / hallucination & 4.12\% & 4.14\% \\
Refusal mention in rationale (heuristic) & 13.83\% & 14.56\% \\
Hard refusal (strict output-level matcher) & 0.42\% & 0.48\% \\
Broad inability phrase (reference only) & 5.89\% & 5.54\% \\
Factual / logical mistakes & 19.26\% & 16.91\% \\
\bottomrule
\end{tabular}
\end{table}

\subsubsection{AUC Experiment Justification}\label{AUC_Just}

Table~\ref{tab:baseline_analysis} shows that SCP attains very high AUC on 
MuSiQue, FEVER, and 2WikiMH against the human-annotated supporting labels. 
However, on MuSiQue this advantage does not translate to a downstream win 
in LLM QA (Table~\ref{tab:llm_qa_master}). We attribute this to a cascading 
issue in the official MuSiQue annotation:

\begin{itemize}
    \item \textbf{Paragraph-level supervision.} MuSiQue provides supporting 
    labels at paragraph granularity. To train a sentence-level reranker we 
    must propagate each paragraph label to every sentence inside it.
    \item \textbf{Distractor leakage after propagation.} A ``supporting 
    paragraph'' typically contains only 1--2 bridge sentences and several 
    filler sentences that are not themselves supporting. The propagation 
    therefore inserts a substantial number of non-supporting sentences into 
    the positive set of our pairwise margin loss.
    \item \textbf{Overfitting to noisy labels.} Training directly on MuSiQue 
    fits a corrupted ranking objective. A model trained on MS MARCO, which 
    carries cleaner sentence-level supervision, transfers back to MuSiQue 
    better than either MuSiQue-trained variant (EM 0.466 vs.\ 0.441 / 0.389 
    / 0.377; see Table~\ref{tab:llm_qa_master}).
\end{itemize}

The broader insight is that \textbf{AUC against human annotations is 
correlated with downstream LLM QA performance, but the two are not 
equivalent}. When the human labels themselves are noisy, fitting them 
more tightly (higher AUC) does not yield more useful pruning. We therefore 
report downstream EM/F1 in Table~\ref{tab:llm_qa_master} as the primary 
metric and treat AUC as a diagnostic.

\begin{table}[h]
\centering
\caption{Extended baseline analysis with ranking and performance gaps.}
\label{tab:baseline_analysis}
\small
\begin{tabular}{lcccccc}
\toprule
\cellcolor{yellow!15}\textbf{Dataset} & \cellcolor{yellow!15}\textbf{Top1} & \cellcolor{yellow!15}\textbf{Top1 AUC} & \cellcolor{yellow!15}\textbf{SCP Rank} & \cellcolor{yellow!15}\textbf{vs BM25} & \cellcolor{yellow!15}\textbf{vs CE} & \cellcolor{yellow!15}\textbf{vs Prov} \\
\midrule
HotpotQA & Provence & 0.872 & 4th & $-$0.008 & $-$0.081 & $-$0.093 \\
MS MARCO & CrossEncoder & 0.881 & 2nd & +0.092 & $-$0.109 & +0.007 \\
MuSiQue & SCP & 0.9381 & 1st & +0.288 & +0.171 & +0.194 \\
FEVER & SCP & 0.927 & 1st & +0.120 & +0.088 & +0.050 \\
2WikiMH & SCP & 0.946 & 1st & +0.296 & +0.083 & +0.197 \\
\bottomrule
\end{tabular}
\end{table}

\subsubsection{Compression Rate Study for SCP on HotpotQA}

Table~\ref{tab:compression_hotpotqa} shows how different compression rates affect SCP's performance.

\begin{table}[h]
\centering
\caption{SCP performance under different compression rates on HotpotQA (GPT-5.4).}
\label{tab:compression_hotpotqa}
\begin{tabular}{llccc}
\toprule
\cellcolor{yellow!15}\textbf{SCP keep\_ratio} & \cellcolor{yellow!15}\textbf{Description} & \cellcolor{yellow!15}\textbf{EM} & \cellcolor{yellow!15}\textbf{F1} & \cellcolor{yellow!15}\textbf{Actual Keep} \\
\midrule
100\% (no compression) & baseline & 0.791 & 0.845 & 1.000 \\
50\% (kr05) & keep half & 0.774 & 0.830 & 0.508 \\
7\% (kr07) & very sparse & 0.629 & 0.683 & 0.066 \\
3\% (kr03) & very sparse & 0.576 & 0.627 & 0.031 \\
\bottomrule
\end{tabular}
\end{table}

\subsubsection{NIAH Supplementary Details}

Tables~\ref{tab:niah_length} and~\ref{tab:niah_depth} show how SCP's recall performance changes under the single-needle setting.

\begin{table*}[t]
\centering
\small
\caption{NIAH Performance Across Different Context Lengths}
\label{tab:niah_length}
\begin{tabular}{lcccccccc}
\toprule
\cellcolor{yellow!15}\textbf{Method} & \cellcolor{yellow!15}\textbf{512} & \cellcolor{yellow!15}\textbf{1K} & \cellcolor{yellow!15}\textbf{2K} & \cellcolor{yellow!15}\textbf{4K} & \cellcolor{yellow!15}\textbf{8K} & \cellcolor{yellow!15}\textbf{16K} & \cellcolor{yellow!15}\textbf{32K} & \cellcolor{yellow!15}\textbf{Average} \\
\midrule
none (100\%) & 1.00 & 1.00 & 1.00 & 1.00 & 1.00 & 1.00 & 1.00 & 1.00 \\
SCP 30\% & 1.00 & 0.98 & 1.00 & 1.00 & 1.00 & 0.98 & 0.98 & 0.991 \\
SCP 50\% & 1.00 & 1.00 & 1.00 & 1.00 & 1.00 & 1.00 & 0.98 & 0.997 \\
SCP 70\% & 1.00 & 1.00 & 1.00 & 1.00 & 1.00 & 1.00 & 1.00 & 1.000 \\
BM25 30-70\% & 1.00 & 1.00 & 1.00 & 1.00 & 1.00 & 1.00 & 1.00 & 1.000 \\
Random 30\% & 0.30 & 0.22 & 0.24 & 0.38 & 0.26 & 0.30 & 0.24 & 0.277 \\
Random 50\% & 0.44 & 0.46 & 0.44 & 0.52 & 0.44 & 0.48 & 0.58 & 0.480 \\
Random 70\% & 0.74 & 0.66 & 0.66 & 0.68 & 0.74 & 0.78 & 0.74 & 0.714 \\
\bottomrule
\end{tabular}
\end{table*}

\begin{table}[h]
\centering
\caption{NIAH Depth Curve Results}
\label{tab:niah_depth}
\begin{tabular}{lccccc}
\toprule
\cellcolor{yellow!15}\textbf{Method} & \cellcolor{yellow!15}\textbf{d=0.1} & \cellcolor{yellow!15}\textbf{d=0.25} & \cellcolor{yellow!15}\textbf{d=0.5} & \cellcolor{yellow!15}\textbf{d=0.75} & \cellcolor{yellow!15}\textbf{d=0.9} \\
\midrule
none & 1.000 & 1.000 & 1.000 & 1.000 & 1.000\\
SCP 30\% & 1.000 & 0.986 & 0.986 & 0.986 & 1.000\\
BM25 30\% & 1.000 & 1.000 & 1.000 & 1.000 & 1.000\\
Random 30\% & 0.257 & 0.271 & 0.214 & 0.314 & 0.329\\
\bottomrule
\end{tabular}
\end{table}

Table~\ref{tab:niah_key_retrieval} reports the performance of reranking and pruning when there are two needles in the context.

\begin{table}[h]
\centering
\caption{NIAH Key Retrieval Results (Two Sentences)}
\label{tab:niah_key_retrieval}
\begin{tabular}{llcccccccc}
\toprule
\cellcolor{yellow!15}\textbf{Method} & \cellcolor{yellow!15}\textbf{KR} & \cellcolor{yellow!15}\textbf{512} & \cellcolor{yellow!15}\textbf{1K} & \cellcolor{yellow!15}\textbf{2K} & \cellcolor{yellow!15}\textbf{4K} & \cellcolor{yellow!15}\textbf{8K} & \cellcolor{yellow!15}\textbf{16K} & \cellcolor{yellow!15}\textbf{32K} & \cellcolor{yellow!15}\textbf{AVG} \\
\midrule
SCP & 30\% & 1.000 & 1.000 & 0.992 & 0.996 & 0.988 & 0.968 & 0.960 & 0.986 \\
SCP & 50\% & 1.000 & 1.000 & 1.000 & 0.992 & 1.000 & 0.992 & 0.988 & 0.996 \\
BM25 & 30\% & 1.000 & 1.000 & 1.000 & 1.000 & 1.000 & 1.000 & 1.000 & 1.000 \\
Random & 30\% & 0.056 & 0.088 & 0.076 & 0.080 & 0.120 & 0.092 & 0.092 & 0.086 \\
Random & 50\% & 0.248 & 0.272 & 0.228 & 0.244 & 0.264 & 0.220 & 0.224 & 0.243 \\
\bottomrule
\end{tabular}
\end{table}

We further test SCP with five needles under different keep ratios to check whether performance degrades. See Table~\ref{tab:niah_mk5_loo_scp}.

\begin{table}[h]
\centering
\small
\caption{5 needle NIAH: all-needle recall under different keep ratios.}
\label{tab:niah_mk5_loo_scp}
\begin{tabular}{lccccc}
\toprule
\cellcolor{yellow!15}\textbf{Method} & \cellcolor{yellow!15}\textbf{Estimator} & \cellcolor{yellow!15}\textbf{KR=0.1} & \cellcolor{yellow!15}\textbf{KR=0.2} & \cellcolor{yellow!15}\textbf{KR=0.3} & \cellcolor{yellow!15}\textbf{Mean} \\
\midrule
LOO & exact & 0.431 & 0.749 & 0.874 & 0.685 \\
SCP & MC=50 & 0.411 & 0.711 & 0.843 & 0.655 \\
SCP & MC=200 & 0.451 & 0.743 & 0.849 & 0.681 \\
SCP & MC=500 & 0.443 & 0.757 & 0.860 & 0.687 \\
SCP & MC=1000 & 0.457 & 0.759 & 0.849 & 0.688 \\
\bottomrule
\end{tabular}
\end{table}

\subsection{Ablations}

\subsubsection{Embedders}

Table~\ref{tab:ablation_embed} compares the performance of SCP when using different embedding mechanisms. The default setting uses Sentence Transformer \cite{reimers2019sentencebertsentenceembeddingsusing} with MiniLM \cite{wang2020minilmdeepselfattentiondistillation}, which is a common choice for sentence-level embeddings. We also test BERT, Word2Vec, TF-IDF with SVD dimensionality reduction, and random embeddings as baselines. 

\begin{table}[h]
\centering
\caption{Ablation: Evaluating SCP Performance under Different Embedding Constraints on HotpotQA}
\label{tab:ablation_embed}
\begin{tabular}{lcccc}
\toprule
\cellcolor{yellow!15}\textbf{Embedding Mechanism} & \cellcolor{yellow!15}\textbf{AUC} & \cellcolor{yellow!15}\textbf{Recall@0.3} & \cellcolor{yellow!15}\textbf{Recall@0.5} & \cellcolor{yellow!15}\textbf{Recall@0.7} \\
\midrule
SentenceTransformer (MiniLM, Default) & 0.779 & 0.642 & 0.800 & 0.906 \\
Word2Vec & 0.750 & 0.630 & 0.790 & 0.895 \\
TF-IDF + SVD & 0.588 & 0.450 & 0.620 & 0.750 \\
Random & 0.507 & 0.380 & 0.520 & 0.650 \\
BERT(Deepsets size: 3.42M > 3.03M) & 0.858 & 0.743 & 0.896 & 0.924 \\
\bottomrule
\end{tabular}
\end{table}

\subsubsection{Attribution Strategies Supplementary Details}\label{sec:LOO}

This section is a controlled study of the LOO attribution \emph{strategy} itself, not a re-implementation of LooComp~\cite{do2026loocompleverageleaveoneoutstrategy}. LooComp couples its LOO signal to a LoRA adapter over a pretrained language model, whereas we hold the value function and embedding backbone fixed and only swap the attribution rule. This isolates the contribution of the attribution principle (LOO vs. Shapley) from confounding factors such as backbone capacity or LoRA fine-tuning.

Table~\ref{tab:loo_vs_shapley_updated} reports the matched-setting comparison between SCP (MC-Shapley) and pure LOO attribution, where both methods use the same learned set-value model and embedding backbone. See Table~\ref{tab:niah_mk5_loo_scp} for the number of Monte-Carlo samples needed to match or exceed the performance of the LOO strategy on the NIAH experiment.

\begin{table}[h]
\centering
\small
\caption{SCP vs. LOO attribution under matched value-function settings.}
\label{tab:loo_vs_shapley_updated}
\begin{tabular}{lccccc}
\toprule
\cellcolor{yellow!15}\textbf{Dataset} & \cellcolor{yellow!15}\textbf{AUC (SCP)} & \cellcolor{yellow!15}\textbf{AUC (LOO)} & \cellcolor{yellow!15}$\Delta$ (SCP-LOO) & \cellcolor{yellow!15}\textbf{R@0.5 (S/L)} & \cellcolor{yellow!15}\textbf{Time s (S/L)} \\
\midrule
HotpotQA (in-domain) & 0.7835 & 0.7789 & 0.0045 &  0.8436 /0.8412 & 1.9 / 0.8 \\
MS MARCO (full) & 0.7729 & 0.7726 & 0.0003 & 0.8157 / 0.8156 & 89.6 / 47.2 \\
MS MARCO (8k cache) & 0.7202 & 0.7175 & 0.0026 & 0.7459 / 0.7456 & 72.8 / 66.6 \\
\bottomrule
\end{tabular}
\end{table}

Using the LOO strategy at inference can be faster and more effective on these benchmarks because of its low variance (SCP relies on Monte-Carlo sampling). In real-world RAG systems, however, semantically duplicated supporting sentences are not guaranteed to be filtered, so the speed gain is obtained at the cost of robustness. LOO mistakes a duplicated supporting sentence for a dummy player because $f(S \cup \{i\}) - f(S) = 0$ when an equivalent sentence already lies in $S$. Using the trained value function we instead compute $v(S \cup \{i\}) - v(S)$ to score a sentence directly. We provide an adversarial example for this strategy below; scores are rounded to one decimal place.

\paragraph{Updated dose + issuer with old-version distractor.} \textit{Q: According to the latest protocol, what is the standard analgesic dose, and who issued it?} The results are reported in Table~\ref{tab:redatk_g3}.

\begin{table}[h]
\centering
\small
\caption{Per-sentence LOO and SCP scores}
\label{tab:redatk_g3}
\begin{tabular}{c p{8.0cm} c r r}
\toprule
\cellcolor{yellow!15}\textbf{\#} & \cellcolor{yellow!15}\textbf{Sentence} & \cellcolor{yellow!15}\textbf{Role} & \cellcolor{yellow!15}\textbf{LOO} & \cellcolor{yellow!15}\textbf{SCP} \\
\midrule
S1  & The old manual states: standard analgesic dose is 5~mg.    & distractor (old)   & -1 & -0.8   \\
S2  & Update notice: the standard dose has been adjusted to 8~mg.& G1$^{\dagger}$     & 0 & 0.7 \\
S3  & The 8~mg standard dose is currently in force.              & G1$^{\dagger}$     & 0.1 & 0.8 \\
S4  & This dose update was issued by Medical Officer Gu~Ning.    & G2$^{\dagger}$     & 0 & 0.5 \\
S5  & Gu~Ning issued the present 8~mg adjustment.                & G2$^{\dagger}$     & -0.1 & 0.5 \\
S6  & The old 5~mg clause has been rescinded.                    & G3                 & 1.2 & 1.3 \\
S7  & Pharmacy stock cap remains 200 vials.                      & noise              & 0 & 0   \\
S8  & Hallway lights enter energy-saving mode after 22:00.       & noise              & -0.1 & -0.1   \\
S9  & A lost access card must be reported within 24~hours.       & noise              & 0 & 0   \\
S10 & Morning meeting moved from 08:00 to 08:30.                 & noise              & 0 & 0   \\
\bottomrule
\end{tabular}
\end{table}

We observe that LOO significantly underestimates duplicated supporting sentences, which makes its selection degenerate toward random. Furthermore, the accuracy of LOO is fixed since the computation is deterministic, whereas SCP can improve with more samples.

\subsubsection{Monte-Carlo Sampling Number} 

Table~\ref{tab:msmarco_sample} shows how the number of Monte-Carlo samples (MC) affects SCP's performance on HotpotQA when trained on MS MARCO. Spped here only refers to the inference time of SCP itself, the embedding time is not included. In practical implementation, we recommend using batch pre-embedding and GPU-parallelized sampling, which can be easily optimized in comparison to the Provence, LLMLingua2 and other methods.

\begin{table}[h]
\centering
\caption{Sampling Number Ablation for SCP trained on MS MARCO, tested on HotpotQA (dev subset, 612 samples).}
\label{tab:msmarco_sample}
\begin{tabular}{llcccccc}
\toprule
\cellcolor{yellow!15}\textbf{Method} & \cellcolor{yellow!15}\textbf{Params} & \cellcolor{yellow!15}\textbf{AUC} & \cellcolor{yellow!15}\textbf{R@0.3} & \cellcolor{yellow!15}\textbf{R@0.5} & \cellcolor{yellow!15}\textbf{R@0.7} & \cellcolor{yellow!15}\textbf{Speed} \\
\midrule
Random & 0 & 0.514 $\pm$ 0.010 & 0.286 & 0.491 & 0.688 & --- \\
BM25 & 0 & 0.787 & 0.714 & 0.832 & 0.899 & 0.4s/612 \\
\midrule
\textbf{SCP (mc10)} & \textbf{3M} & \textbf{0.719} & 0.578 & 0.792 & 0.894 & 1.9s/612 \\
\textbf{SCP (mc50)} & \textbf{3M} & \textbf{0.750} & 0.642 & 0.814 & 0.906 & 2.5s/612 \\
\textbf{SCP (mc150)} & \textbf{3M} & \textbf{0.752} & 0.633 & 0.820 & 0.922 & 4.6s/612 \\
CrossEncoder & 22M & 0.860 & 0.851 & 0.910 & 0.943 & 13.2s/612 \\
Provence & 149M & 0.872 & 0.875 & 0.944 & 0.980 & 36.3s/612 \\
\bottomrule
\end{tabular}
\end{table}

\subsection{Prompt Usage in Inference}

Table~\ref{tab:prompt-usage} lists the system prompts and generation settings used for LLM-side evaluation across QA datasets. HotpotQA and 2WikiMultiHopQA share the same system prompt; MuSiQue uses a stricter variant with a tighter output budget; CL-Bench reuses the dataset-released \texttt{messages} verbatim.

\begin{table}[h]
\centering
\caption{System prompts and generation settings across QA datasets.}
\label{tab:prompt-usage}
\footnotesize
\begin{tabular}{p{1.6cm}p{7.8cm}cc}
\toprule
\cellcolor{yellow!15}\textbf{Dataset} & \cellcolor{yellow!15}\textbf{System Prompt} & \cellcolor{yellow!15}\textbf{max\_tok.} & \cellcolor{yellow!15}\textbf{temp.} \\
\midrule
HotpotQA & You are a question answering assistant. Answer the question based on the provided context. Give a concise answer --- typically 1--5 words for factoid questions. Do NOT explain. & 64 & 0.0 \\
2WikiMHQA & Same as HotpotQA. & 64 & 0.0 \\
MuSiQue & You are a question answering assistant. Answer the question based on the provided context. Give a concise answer --- a short phrase only (1--5 words). Do NOT explain. & 32 & 0.0 \\
CL-Bench & Dataset-released \texttt{messages} used verbatim; no additional system prompt. & default & default \\
\bottomrule
\end{tabular}
\end{table}

\section{Justifications, Hypotheses and More Future Works} 

\subsection{Why 3M parameters?}\label{sec:why-3M}

We intentionally design the SCP model to be substantially smaller than competing models in order to show that the performance gains do not stem from raw model capacity alone but also from the evaluation paradigm. With only 3M parameters, a lightweight model can still capture useful interdependencies among context components under the Shapley-based training objective. Our rationale is to decouple attribution complexity from the neural representation: we delegate the attribution computation to the Monte-Carlo sampling mechanism, while the neural value function approximates holistic subset values---a simpler prediction task. Practically, this design reduces inference time, since the value function only needs lightweight forward passes and the Monte-Carlo aggregation can be parallelized on the GPU.

\subsection{Impact of the Convexity Constraint}\label{sec:conv-const}

As described in appendix A, modeling supermodular valuations enforces a coherent mathematical
hierarchy---the foundation for analysing the ``Landscape of Context''. To empirically transition from
heuristic sets to a rigorous core, we implement the structurally restricted \texttt{ConvexDeepSetsV2}.

Imposing continuous convexity bounds the hypothesis space. Historically this can reduce
predictive accuracy on elementary lexical tasks (e.g., a slight degradation on MS MARCO). However, on
combinatorial tasks (HotpotQA, MuSiQue), the convex prior is more aligned with our conjecture that
some bridging components exhibit complementarity (e.g., $v(A \cup B) > v(A) + v(B)$). In this sense,
enforcing convexity acts as an inductive bias for capturing synergy in multi-hop chains and provides
empirical support for the proposed hierarchical Context-Landscape perspective.

Table~\ref{tab:convexv2_comparison} reveals nuanced trade-offs between the unconstrained
SCP and ConvexV2. Rather than positioning ConvexV2 as a universally superior variant, these results
illuminate a fundamental alignment between architectural inductive biases and the underlying
combinatorial structure of different reasoning tasks.

\begin{table}[h]
\centering
\caption{ConvexV2 vs. Original SCP: Best AUC Comparison}
\label{tab:convexv2_comparison}
\begin{tabular}{lccc}
\toprule
\cellcolor{yellow!15}\textbf{Dataset} & \cellcolor{yellow!15}\textbf{Original SCP} & \cellcolor{yellow!15}\textbf{ConvexV2} & \cellcolor{yellow!15}$\Delta$ \\
\midrule
HotpotQA & 0.7787 & \textbf{0.7945} & +0.016 \\
MS MARCO & \textbf{0.7622} & 0.6966 & $-$0.066 \\
MuSiQue & \textbf{0.9459} & 0.9345 & $-$0.0114 \\
FEVER & 0.9279 & \textbf{0.9327} & +0.005 \\
\bottomrule
\end{tabular}
\end{table}

\subsection{Ablation Explanation for Feature Extraction}\label{sec:feat-ext}

Notably, downgrading the embedding input from a deep bidirectional contextual model (SentenceTransformer)
to a shallow lexical model (Word2Vec) induces only a $\sim$3\% drop in AUC on HotpotQA while retaining
strong multi-hop performance. This suggests that SCP's effectiveness is not solely due to a
powerful sentence embedding model, and that the cooperative aggregation mechanism (the $\rho$ layer and
the Shapley allocations) plays an important role.

Nevertheless, a baseline level of embedding quality is still required to capture basic semantic relevance,
as evidenced by the substantial performance drop when using TF-IDF + SVD or random embeddings.
This indicates that while SCP's Shapley-based training objective contributes to its robustness,
it still relies on a certain level of semantic encoding in the input features to model
context dependencies effectively. The ablation supports the view that at the in-sentence-token level the embedding
requires a fine-grained understanding of the query and the context---especially the position of each token
in the sentence---while at the in-document-sentence level a more coarse-grained evaluation of the overall
semantic relevance of the sentence to the query is sufficient for SCP to capture the interdependencies
among sentences.

\subsection{Architecture Justification for the Permutation-Invariant Design and training-inference alignment}\label{sec:per_inv}

In Natural Language Processing, positional encodings and order information are usually presumed to be necessary. However, although positional encodings are essential at smaller scales, in context pruning---especially in the long-context regime---the well-known ``lost in the middle'' phenomenon is widely reported, and set modeling has proven effective in many long-context scenarios, particularly in RAG. In RAG, retrieved sentences are not organized in any fixed order, and the retrieval system may return sentences in arbitrary order. We hypothesize that sentences are not simply organized in a sequence but rather form a complex logical hierarchy, which we call the ``Landscape of Context''.

\paragraph{Design Rationale: Why Pairwise Ranking Suffices for Coarse-Grained Pruning}
We deliberately employ a pairwise margin ranking loss rather than a regression-based objective for three reasons rooted in the practical realities of RAG context pruning. 
First, \textbf{Dataset Annotation Constraints}: Standard RAG datasets (e.g., HotpotQA, MuSiQue, MS MARCO) provide only binary supporting-sentence labels; human annotators mark whether a sentence is relevant, but not its degree of relevance relative to other supporting sentences. Obtaining such fine-grained comparative labels would require prohibitively expensive additional annotation. 
Second, \textbf{Task-level Requirements}: Our goal is not to select the single most important sentence, but to filter out non-supporting sentences while retaining as many supporting sentences as possible. This is a coarse-grained pruning objective: by operating at a conservative compression ratio, we prioritize recall over precision at the pruning stage, ensuring that even if the ranking among supporting sentences is imperfect, critical evidence is not lost. The top-$K$ pruning operation is mathematically a ranking truncation: it depends only on the relative ordering $\phi_{(1)} \ge \phi_{(2)} \ge \dots \ge \phi_{(n)}$, not on absolute magnitudes. Formally, for any strictly increasing monotonic transformation $f$, the top-$K$ set satisfies $\{i : \phi_i \ge \phi_{(K)}\} = \{i : f(\phi_i) \ge f(\phi_{(K)})\}$. Since Shapley values are homogeneous of degree one in the value function—$\phi_i(\alpha v) = \alpha \cdot \phi_i(v)$—the absolute scale is inherently arbitrary and only relative comparisons are structurally meaningful in cooperative game theory. 
Third, \textbf{Generalization across LLMs}: The absolute importance of a sentence is inherently dependent on the downstream LLM's internal knowledge and reasoning patterns. Attempting to learn LLM-specific absolute importance scores would harm cross-model generalization. A relative ranking objective, by contrast, learns an LLM-agnostic ordering that transfers robustly across diverse generators.

\subsection{Top-K Justification}\label{sec:top-k}

While Top-K and Average-Threshold are direct, empirically effective implementations, we hope that future work will explore more sophisticated pruning criteria that leverage the full distribution of Shapley values, such as clustering-based methods or dynamic thresholding based on the variance of attributions. The mathematical flexibility of Shapley values supports a wide range of pruning strategies beyond simple cutoffs, which can be tailored to specific downstream tasks and computational constraints, as much prior work has explored. In our setting, however, given the low parameter count of our value function and the use of sentence-level pruning, the simple Top-K and Average-Threshold methods already achieve competitive performance across the evaluated datasets, since SCP is designed primarily as a preliminary reranker and pruner for the retrieved context, after which the downstream LLM can further filter and use the pruned context to generate the final answer.

\subsection{The Landscape of Context: Scale Change May Means Paradigm Shift}

In most popular NLP settings, the necessity of positional encoding is taken for granted.
Historically, the absence of sequence awareness in the original self-attention mechanism
(e.g., its inability to distinguish ``dog chases cat'' from ``cat chases dog'') prompted the adoption of
positional embeddings in Transformer architectures. However, as context length scales exponentially
in the LLM era, the prevalent ``lost in the middle'' phenomenon suggests that models struggle
to allocate attention uniformly across long, linear sequences.

This awareness is not new. In earlier years, researchers developed techniques such as Latent Dirichlet Allocation (LDA) \cite{10.5555/944919.944937} to classify documents into a limited set of topics, Set Transformers \cite{lee2019settransformerframeworkattentionbased} to process unordered sets, and Deep Sets to model permutation-invariant functions. The advent of the Transformer architecture \cite{vaswani2023attentionneed}, however, encouraged the assumption that any context could be effectively processed through an adapted attention mechanism. As a result, work that adapts attention mechanisms has received the bulk of attention, while other theoretical perspectives have been overlooked.

We propose Shapley Context Pruning (SCP) as an attempt to construct an analytical framework based on cooperative game theory, a field thoroughly studied in economics and mathematics. The theory provides a new perspective on the structure of context and the interactions between sentences, which can be used to guide the design of more effective context-reranking methods and to provide a principled tool for analysing the structure of the context. This raises the question of whether macro-level sequence order may be fundamentally less important than micro-level token order in certain contexts, especially in RAG systems where retrieved passages are returned in non-deterministic orders without strict sequential causality.

This leads to a paradigm-shifting hypothesis:
\textbf{In extreme-long-context Retrieval-Augmented Generation (RAG) \cite{lewis2021retrievalaugmentedgenerationknowledgeintensivenlp} settings, macro-level
sequence order is fundamentally less important than micro-level token order, though sequence modeling in sentence level is essential.} While positional
information is undeniably crucial intra-sentence to preserve semantic and syntactic integrity, inter-sentence organization in RAG systems and complex user prompts is often an unordered ``Bag-of-Evidence''. Retrieved passages are frequently provided in non-deterministic orders without strict sequential causality.

The hypothesis suggests a possible paradigm: sequential modeling and unordered set modeling are two methods operating at different scales. The former is more suitable for modeling token-level interactions within a sentence, while the latter is more suitable for modeling sentence-level interactions within a long context. Recent work such as Set-Encoder \cite{Schlatt_2025} applies a Set Transformer at the document level to implement cross-document attention while still using BERT-based embeddings, researchers start to explore setwise modeling rather than simple pairwise, pointwise, listwise modeling. Combining set-based methods at the larger scale with position-aware methods at the smaller scale may be a more effective way to model the context, particularly in the RAG scenario.

We therefore argue against the forced sequential projection of long contexts. We instead propose the ``Landscape of Context''~\ref{sec:landscape}, a theoretical abstraction that conceptualises context not as a flat sequence but as a hierarchical topology of subsets.

We propose this not as a closed problem but as a conceptual genesis. We leave the rigorous mathematical formulation of this landscape's exact boundaries, its adaptation to non-convex empirical settings, and its broader applications in LLM alignment as open invitations to the research community. By liberating context modeling from the strict inductive biases of linear positional embeddings, we hope that future theoretical work will explore the ``Landscape of Context'' architecture in depth and its fundamental implications for next-generation context engineering.

\subsection{More Future Work and Hypothesis}\label{sec:more_future}

We believe that, as context grows larger and more complex, SCP can offer a less fine-grained (compared with LLM- and BERT-based methods) but effective, fast, and interpretable pruning method for preliminary context compression in RAG systems and edge-device deployment, saving tokens and reducing the LLM-hallucination risk induced by distractors. We suggest that future work explore a multi-layer context-pruning pipeline, where SCP serves as a fast and interpretable first stage that filters out the most irrelevant components, followed by more fine-grained methods (e.g., BERT-based) for further pruning. Such a multi-layer approach can balance efficiency and precision, leveraging the strengths of different pruning techniques at different stages of the pipeline. SCP can also be deployed on local devices to perform preliminary context compression before sending the pruned context to the cloud for LLM inference, saving bandwidth and reducing latency in real-world applications.

\paragraph{Towards Incremental Context Engineering.} The current SCP framework computes Shapley values in batch. In real-world RAG, however, retrieved sentences arrive incrementally, and recomputing all attributions from scratch is wasteful. A promising direction is to exploit the structural properties of convex games to derive \emph{incremental update rules}. Future work could develop online algorithms that maintain a laminar tree dynamically, inserting new nodes and rebalancing only along the affected path rather than reconstructing the entire hierarchy.

\paragraph{Forest Models for Multi-Topic Contexts.}\label{para:solving_non_convex} A single convex game assumes that the entire context is one large cooperative structure. Real documents, however, typically contain multiple semantically disjoint topics or reasoning chains. We propose modelling such contexts as a \emph{forest of convex games}: each tree corresponds to an independent topic or multi-hop chain (a local convex game), while the forest captures the weak or negligible cross-topic interactions. Under the idealised assumption that cross-topic synergy vanishes, the value function becomes block-additive across a partition $\{N_j\}$ of $N$:
\[
v(S) \;=\; \sum_{j} v_j\bigl(S \cap N_j\bigr), \qquad \forall\, S \subseteq N.
\]
To see that attributions can then be computed locally, define the extension $\tilde{v}_j(S) = v_j(S \cap N_j)$ for all $S \subseteq N$, so that $v = \sum_j \tilde{v}_j$. By Additivity,
\[
\phi_i(v) \;=\; \sum_{j} \phi_i(\tilde{v}_j).
\]
For any block $N_j$ not containing $i$, player $i$ is a dummy in $\tilde{v}_j$ (adding $i$ does not change $S \cap N_j$); hence $\phi_i(\tilde{v}_j)=0$ by the Dummy-Player axiom. It remains to relate $\phi_i(\tilde{v}_{j^*})$ to the block-level Shapley value for the unique block $N_{j^*} \ni i$. Expanding the definition, consider any $S \subseteq N \setminus \{i\}$ and define $T = S \cap N_{j^*}$ and $U = S \cap (N \setminus N_{j^*})$, so that $S = T \cup U$ and $T \cap U = \emptyset$.
\begin{align*}
\phi_i(\tilde{v}_{j^*})
&= \sum_{S \subseteq N \setminus \{i\}} \frac{|S|!\,(|N|-|S|-1)!}{|N|!}
   \bigl[\tilde{v}_{j^*}(S\cup\{i\}) - \tilde{v}_{j^*}(S)\bigr] \\
&= \sum_{T \subseteq N_{j^*}\setminus\{i\}}
   \Biggl(
     \sum_{U \subseteq N \setminus N_{j^*}}
     \frac{(|T|+|U|)!\,(|N|-|T|-|U|-1)!}{|N|!}
   \Biggr)
   \bigl[v_{j^*}(T\cup\{i\}) - v_{j^*}(T)\bigr].
\end{align*}
The inner sum over $U$ satisfies the combinatorial identity
\[
\sum_{U \subseteq N \setminus N_{j^*}}
\frac{(|T|+|U|)!\,(|N|-|T|-|U|-1)!}{|N|!}
\;=\;
\frac{|T|!\,(|N_{j^*}|-|T|-1)!}{|N_{j^*}|!},
\]
which collapses the global Shapley formula to the local one:
\[
\phi_i(\tilde{v}_{j^*})
\;=\;
\sum_{T \subseteq N_{j^*}\setminus\{i\}}
\frac{|T|!\,(|N_{j^*}|-|T|-1)!}{|N_{j^*}|!}
\bigl[v_{j^*}(T\cup\{i\}) - v_{j^*}(T)\bigr]
\;=\;
\phi_i^{N_{j^*}}(v_{j^*}).
\]
Thus $\phi_i(v) = \phi_i^{N_{j^*}}(v_{j^*})$: each player's global attribution equals its attribution within its own block. This decomposition holds for \emph{any} block-additive game, convex or not; convexity is only needed to guarantee that each block induces a laminar hierarchy with non-negative Harsanyi dividends. The deeper challenge is therefore three-fold: (i)~enforcing block-additivity, which is unlikely to emerge from standard training and must be imposed architecturally or via regularization; (ii)~preserving convexity within each block so that the resulting trees remain interpretable; and (iii)~discovering the partition $\{N_j\}$ from data. We leave the joint realization of these three desiderata as an exciting direction for future work.

\paragraph{Transfer Learning and Multi-objective Learning.} The additivity of Shapley values raises a natural question: can two value functions trained on different datasets or under different loss objectives be directly composed into a single attribution? When such direct composition fails, an equally informative question is to characterize the gap---the representational discrepancy between the constituent value functions, expressed through their induced Shapley values. A principled answer would clarify when transfer and multi-objective training of context value functions reduce to a Shapley-additive combination, and when they instead require explicit interaction modeling between the underlying objectives.

\paragraph{Designing Context Management System} The ``Landscape of Context'' perspective (Appendix~\ref{sec:landscape}) suggests a new design for context management systems in RAG. Instead of treating retrieved sentences as a flat sequence, we may organize them into the hierarchical structure induced by the laminar families of a convex context game, with Shapley values and Harsanyi dividends quantifying per-node synergy. This may allow a system to dynamically insert newly retrieved sentences into the existing hierarchy, evaluating and pruning them based on their localized contribution within the tree rather than via a flat ranking. We expect the main practical benefit to be amortization via decomposition, perhaps utilizing additivity of Shapley values across the forest of value functions $\{v_i\}$ defined in Appendix~\ref{sec:landscape}.

\section{Simulated Real World Case Study for Context Reranking: A Research Scenario in Cognitive Neuroscience}

We construct a real-world case in cognitive neuroscience to demonstrate how context attribution behaves under multimodal data fusion. The scenario involves an fMRI--EEG concurrent experiment on working-memory load. This case generated by GPT-5.4 \cite{openai2026gpt54} is based on the authors' real-world experience in neuroimaging research, and is designed to illustrate the practical challenges and decision-making processes that arise during the preparation, acquisition, and analysis phases of a complex multimodal experiment. The dialogue captures the interactions between a principal investigator (PI), a PhD student, a research assistant (RA), and an engineer, highlighting how technical considerations, data quality issues, and methodological choices are navigated in a real-world research setting. The case study is structured into three phases: preparation, acquisition and incident, and analysis, with each phase containing specific dialogues that reflect the critical moments in the experimental workflow.

Table~\ref{tab:casestudy_neuro_setup} records the preparation phase; Table~\ref{tab:casestudy_neuro_acq} covers acquisition and the incident; Table~\ref{tab:casestudy_neuro_analysis} contains the methodological discussion and research question.

\begin{table}[t]
\centering
\small
\caption{Neuroimaging case study---Phase 1: Preparation.}
\label{tab:casestudy_neuro_setup}
\begin{tabular}{p{2.2cm}p{11.5cm}}
\toprule
\cellcolor{yellow!15}\textbf{Role} & \cellcolor{yellow!15}\textbf{Dialogue (condensed)} \\
\midrule
\textbf{PI} & We need clean $\alpha$ and $\theta$ signals for the haemodynamic--neuroelectric coupling analysis. Data quality is the priority. \\
\textbf{PhD Student} & Subject screening is complete. Shall we use the 32-channel wet-electrode cap? Impedance will be more stable against RF interference, though prep time extends to 40~min. \\
\textbf{PI} & Use wet electrodes. Time is not the issue. \\
\textbf{RA (Zhang)} & Impedances are below 5\,k$\Omega$. Cz, Pz, and Fz are stable at $\sim$3\,k$\Omega$, but Oz hovers at 8\,k$\Omega$---probably insufficient conductive paste through the hair. \\
\textbf{PhD Student} & Re-apply paste at Oz with a syringe; follow the international 10--20 system strictly. Any misplacement in the occipital region will bias the visual working-memory localisation. \\
\textbf{Engineer} & Scanner ready. Eddy-current compensation and gradient warm-up are done. Magnet centre frequency drifted from 123.25 to 123.31~MHz; I re-ran auto-shimming and now homogeneity is within 5~ppm. Suggest a localiser first. \\
\textbf{PI} & Run the three-plane localiser. Then do the audio--visual reaction test; remind the subject to press the emergency call button if they feel any discomfort. \\
\textbf{RA} & Subject briefed on the 2-back task. Optical fibre response pads are fitted (EM-safe). Practice rounds completed. \\
\bottomrule
\end{tabular}
\end{table}

\begin{table}[t]
\centering
\small
\caption{Neuroimaging case study---Phase 2: Acquisition and incident.}
\label{tab:casestudy_neuro_acq}
\begin{tabular}{p{2.2cm}p{11.5cm}}
\toprule
\cellcolor{yellow!15}\textbf{Role} & \cellcolor{yellow!15}\textbf{Dialogue (condensed)} \\
\midrule
\textbf{PhD Student} & Heart rate 75, BP 120/80. For sync, I will use an opto-isolated pulse triggered by the MRI TR to start EEG acquisition---millisecond precision. EEG sampling rate: 2000\,Hz? \\
\textbf{PI} & 2000\,Hz. We need high sampling for time--frequency and phase--amplitude coupling analyses. Record event markers for task onset, stimulus onset, and response times. \\
\textbf{Engineer} & Localiser done. Starting T1-weighted 3D MPRAGE: TR=2300\,ms, TE=2.32\,ms, flip angle 8$^\circ$, 1\,mm isotropic. Now EPI: TR=2000\,ms, TE=30\,ms, 37 axial slices, 3\,mm thickness, FOV=220\,mm, 360 volumes, 12-min N-back task. \\
\textbf{RA} & Stimulus sequence running. Mean RT $\approx$650\,ms, accuracy $\approx$85\%. EEG stream stable. BCG artifact visible ($\pm$50\,$\mu$V), especially at T-wave. \\
\textbf{PhD Student} & Record raw data; remove artifacts offline. Monitor head motion---halt if translation $>$1\,mm or rotation $>$1$^\circ$. \\
\textbf{PI} & We target DLPFC and PPC activation under high WM load (3-back), and the relationship between BOLD and $\theta$--$\gamma$ coupling. No compromises on technical details. \\
\textbf{RA} & \textbf{Incident:} Subject reports claustrophobia at volume 210. HR spikes to 110. Severe motion artifact in both fMRI and EEG. \\
\textbf{PI} & Pause immediately. Comfort the subject. If they cannot continue, save data up to the current TR. \\
\textbf{PhD Student} & Stop sent. EEG ``interrupt'' marker inserted. Vitals safe. We have $\sim$7\,min of valid data: two full repeats of 1-back and 3-back blocks each. \\
\textbf{PI} & Sufficient for preliminary GLM. Export fMRI as NIfTI (dcm2niix, with slice-timing info), EEG as EDF, and alignment timestamps as CSV. \\
\textbf{Engineer} & Three EEG versions will be exported: raw, 0.1--100\,Hz band-pass, and 50-Hz notch-filtered. \\
\bottomrule
\end{tabular}
\end{table}

\begin{table}[t]
\centering
\small
\caption{Neuroimaging case study---Phase 3: Analysis discussion and research question.}
\label{tab:casestudy_neuro_analysis}
\begin{tabular}{p{2.2cm}p{11.5cm}}
\toprule
\cellcolor{yellow!15}\textbf{Role} & \cellcolor{yellow!15}\textbf{Dialogue (condensed)} \\
\midrule
\textbf{RA} & For multimodal fusion, should we extract EEG power envelopes via wavelet transform and cross-correlate with BOLD, or use DCM for effective connectivity? \\
\textbf{PI} & Use sliding-window correlation plus wavelet coherence first to find temporal coupling, then DCM for mechanistic interpretation. The key challenge is the temporal resolution mismatch---BOLD is seconds, EEG is milliseconds. \\
\textbf{PhD Student} & We can down-sample EEG to 0.5\,Hz to match BOLD, or deconvolve BOLD to infer neural events. But the latter assumes a fixed neurovascular coupling function across regions. \\
\textbf{PI} & Exactly. That is the core question: \textit{Does neurovascular coupling in DLPFC dynamically change with increasing working-memory load, thereby affecting fMRI-based network inference?} \\
\midrule
\multicolumn{2}{>{\columncolor{yellow!15}}l}{\textbf{Research Question}} \\
\midrule
\multicolumn{2}{p{13.7cm}}{%
Design a fusion pipeline that simultaneously tests (i) whether DLPFC BOLD increase is temporally coupled with $\theta$-band (4--8\,Hz) power rise when WM load increases from 1-back to 3-back, and (ii) whether this coupling strength predicts individual behavioural performance (RT and accuracy). Specify preprocessing, feature extraction, statistical testing, and correction for the temporal-resolution mismatch and multiple comparisons.
} \\
\bottomrule
\end{tabular}
\end{table}

\paragraph{SCP pruning diagnosis on the neuroimaging case.}
When the full dialogue (29 sentences) is fed to SCP, the model retains 12 sentences (40\% compression). Table~\ref{tab:neuro_scp_diagnosis} summarises the pruning outcome.

\begin{table}[h]
\centering
\small
\caption{SCP attribution ranking on the neuroimaging case study.}
\label{tab:neuro_scp_diagnosis}
\begin{tabular}{clcc}
\toprule
\cellcolor{yellow!15}\textbf{Rank} & \cellcolor{yellow!15}\textbf{Speaker} & \cellcolor{yellow!15}\textbf{Shapley} & \cellcolor{yellow!15}\textbf{Kept?} \\
\midrule
1  & PI   & $+$0.080 & \greencmark \; Emergency stop \\
2  & RA      & $+$0.050 & \greencmark \; Subject distress alert \\
3  & RA      & $+$0.010 & \greencmark \; Post-incident check \\
4  & Engineer & $-$0.070 & \greencmark \; Scanner status mid-run \\
5  & PhD     & $-$0.090 & \greencmark \; Experiment log \\
11 & PI   & $-$0.120 & \greencmark \; Sliding-window / wavelet \\
12 & RA     & $-$0.129 & \redxmark \; DCM vs.~wavelet query \\
25 & PI   & $-$0.184 & \redxmark \; GLM mention \\
\bottomrule
\end{tabular}
\end{table}

The result reveals a clear pattern: SCP, trained on MS MARCO, strongly favours \textit{action-oriented} sentences (emergency procedures, equipment checks) while deprioritising \textit{methodological discussion} (DCM, GLM, wavelet transforms). This aligns with the observation in Section~\ref{sec:feat-ext} that the value function is sensitive to concrete operational cues but less attuned to domain-specific methodological terminology. The retained sliding-window sentence (rank~11) survives only because it explicitly mentions ``temporal resolution mismatch''---a phrase structurally similar to the technical details seen in MS MARCO passages. This case therefore illustrates both a strength (robust handling of procedural urgency) and a limitation (methodological nuance is pruned) of the current SCP model, motivating the multi-layer pipeline proposed in Section~\ref{sec:top-k}.

\clearpage

\end{document}